\RequirePackage{fix-cm}
\let\oldvec\vec
\documentclass[twocolumns, natbib]{svjour3}
\let\vec\oldvec

\newif\ifallannexe
\allannexetrue 

\newif\ifworkingpaper
\workingpapertrue 

\usepackage{amsmath, amssymb, amsfonts,mathtools,stmaryrd}
\usepackage{dsfont, bm}
\usepackage[english]{babel}
\usepackage{pgffor}

\usepackage[dvipsnames]{xcolor}
\newcommand\myshade{85}
\colorlet{mylinkcolor}{violet}
\colorlet{mycitecolor}{YellowOrange}
\colorlet{myurlcolor}{Aquamarine}
\definecolor{myblue}{RGB}{18,75,126}
\definecolor{myorange}{RGB}{241,105,19}
\definecolor{myred}{RGB}{179,0,0}

\usepackage[colorlinks,
	linkcolor=myblue,
	citecolor=myorange,
	urlcolor=Aquamarine!\myshade!black,
	backref=page
]{hyperref}

\usepackage[normalem]{ulem}
\usepackage{soul}

\usepackage{graphicx}

\usepackage{booktabs, multirow, multicol}
\usepackage{algorithmic}
\usepackage[algoruled, linesnumbered]{algorithm2e}
\usepackage{url}

\usepackage{tikz}

\definecolor{green_commentary}{RGB}{35,139,69}

\SetCommentSty{mycommfont}

\def\bse{{\boldsymbol{e}}}

\def\bsn{{\boldsymbol{n}}}

\def\bsq{{\boldsymbol{q}}}

\def\bsu{{\boldsymbol{u}}}
\def\bsv{{\boldsymbol{v}}}

\def\bsz{{\boldsymbol{z}}}

\def\bsY{{\boldsymbol{Y}}}
\def\bsZ{{\boldsymbol{Z}}}

\def\bfe{{\mathbf{e}}}

\def\bfp{{\mathbf{p}}}

\def\bfs{{\mathbf{s}}}

\def\bfu{{\mathbf{u}}}
\def\bfv{{\mathbf{v}}}

\def\bfx{{\mathbf{x}}}
\def\bfy{{\mathbf{y}}}
\def\bfz{{\mathbf{z}}}

\def\bfA{{\mathbf{A}}}
\def\bfB{{\mathbf{B}}}
\def\bfC{{\mathbf{C}}}

\def\bfH{{\mathbf{H}}}

\def\bfN{{\mathbf{N}}}

\def\bfP{{\mathbf{P}}}

\def\bfR{{\mathbf{R}}}

\def\bfU{{\mathbf{U}}}

\def\bfW{{\mathbf{W}}}
\def\bfX{{\mathbf{X}}}
\def\bfY{{\mathbf{Y}}}
\def\bfZ{{\mathbf{Z}}}

\def\calU{{\mathcal{U}}}

\def\calS{{\mathcal{S}}}

\def\calA{{\mathcal{A}}}
\def\calB{{\mathcal{B}}}

\def\calH{{\mathcal{H}}}
\def\calI{{\mathcal{I}}}
\def\calJ{{\mathcal{J}}}
\def\calK{{\mathcal{K}}}

\def\calN{{\mathcal{N}}}
\def\calO{{\mathcal{O}}}

\def\calS{{\mathcal{S}}}

\def\calU{{\mathcal{U}}}

\graphicspath{{figures/}}

\newcommand{\headTab}[1]{ \textnormal{\textbf{#1}} }

\usepackage{tikz}

\newcommand{\Nobs}{N}
\newcommand{\nobs}{n}
\newcommand{\dimcoef}{K}
\newcommand{\indcoef}{k}
\newcommand{\dimobs}{D}
\newcommand{\indobs}{d}

\newcommand{\MATobs}{\bfY}
\newcommand{\Vobs}{\bfy}
\newcommand{\obs}[1]{y_{#1}}

\newcommand{\MATcoef}{\bfX}
\newcommand{\Vcoef}{\bfx}
\newcommand{\coef}[1]{x_{#1}}

\newcommand{\MATibp}{\bfZ}
\newcommand{\Vibp}{\bfz}
\newcommand{\ibp}[1]{z_{#1}}

\newcommand{\MATfac}{\bfP}
\newcommand{\Vfac}{\bfp}

\newcommand{\Vnoise}{\bfe}

\newcommand{\noisestd}{\sigma}
\newcommand{\noisevar}{\noisestd^2}

\newcommand{\ibpkn}{z_{\indcoef, \nobs}}
\newcommand{\coefkn}{\coef{\indcoef, \nobs}}

\newcommand{\given}{|}
\newcommand{\prob}{ \mathrm{p} }
\newcommand{\Prob}{ \mathrm{P} }
\newcommand{\qprob}{ \mathrm{q} }

\newcommand{\stiefel}[2]{\calS_{#1}^{#2}}
\newcommand{\zellner}{\delta}

\newcommand{\ibpparam}{\alpha}

\newcommand{\transp}{^T}
\newcommand{\Esp}{{\rm I\kern-.3em E}}

\newcommand{\R}{\mathbb{R}}

\newcommand{\indfunc}[2]{\mathds{1}_{#1}{\left(#2\right)}}
\newcommand{\indmat}{\mathds{I}}

\newcommand{\paramvect}{\boldsymbol{\theta}}
\newcommand{\hypervect}{\boldsymbol{\phi}}

\newcommand{\scalarProdSquare}[2]{\big(#1\transp#2\big)^2}

\DeclareMathOperator{\trace}{tr}
\DeclareMathOperator{\etr}{etr}

\def\thetabf{\boldsymbol{\theta}}

\newcommand{\zellnervar}{ \zellner^2 }
\newcommand{\mkn}{ m_{\indcoef, -\nobs} }

\newcommand{\prodYPscalar}{ \left\langle \Vfac_\indcoef,  \Vobs_\nobs \right\rangle^2 }

\renewcommand{\epsilon}{\varepsilon}
\newcommand{\Khat}{\widehat{\dimcoef}}
\newcommand{\Kkolmo}{\Khat_{\mathrm{KS}}}

\begin{document}


\title{Bayesian nonparametric Principal Component Analysis}

\author{Cl\'ement Elvira \and
        Pierre Chainais  \and
        Nicolas Dobigeon 
}


\institute{C. Elvira \and P. Chainais \at
              Univ. Lille, CNRS, Centrale Lille, UMR 9189 - CRIStAL - Centre de Recherche en Informatique Signal et Automatique de Lille, F-59000 Lille, France \\
              \email{\{clement.elvira, pierre.chainais\}@centralelille.fr}           
           \and
           N. Dobigeon \at
              University of Toulouse, IRIT/INP-ENSEEIHT, CNRS, 2 rue Charles Camichel, BP 7122, 31071 Toulouse cedex 7, France \\
              \email{nicolas.dobigeon@enseeiht.fr}
}

\date{}

\maketitle


\begin{abstract}
	Principal component analysis (PCA) is very popular to perform dimension reduction.
	The selection of the number of significant components is essential but often based on some practical heuristics depending on the application.
	Only few works have proposed a probabilistic approach able to infer the number of significant components.
	To this purpose, this paper introduces a Bayesian nonparametric principal component analysis (BNP-PCA).
	The proposed model projects observations onto a random orthogonal basis which is assigned a prior distribution defined on the Stiefel manifold.
	The prior on factor scores involves an Indian buffet process to model the uncertainty related to the number of components.
	The parameters of interest as well as the nuisance parameters are finally inferred within a fully Bayesian framework via Monte Carlo sampling.
	A study of the (in-)consistence of the marginal maximum a posteriori estimator of the latent dimension is carried out.
	A new estimator of the subspace dimension is proposed. Moreover, for sake of statistical significance, a Kolmogorov-Smirnov test based on the posterior distribution of the principal components is used to refine this estimate.
	The behaviour of the algorithm is first studied on various synthetic examples. Finally,  the proposed BNP dimension reduction approach is shown to be easily yet efficiently coupled with clustering or latent factor models within a unique framework.
\end{abstract}






\section{Introduction}\label{sec:introduction}

Dimension reduction (DR) is an ubiquitous preprocessing step in signal processing and statistical data analysis.
It aims at finding a lower dimensional subspace explaining a set of data while minimizing the resulting loss of information. Related interests are numerous, e.g., reducing the impact of noise, data storage, computational time.

Principal component analysis (PCA) permits DR by projecting observations onto a subset of orthonormal vectors. It provides an elegant solution to DR by looking for a $\dimcoef$-dimensional representation of a dataset $\MATobs=\left[\Vobs_{1},\ldots,\Vobs_{N}\right]$ with $\Vobs_{n} \in \mathbb{R}^{\dimobs}$ in an orthonormal basis, referred to as principal components. Given $\dimcoef$, the $\dimcoef$-dimensional subspace spanned by these principal components is supposed to minimize the quadratic reconstruction error of the dataset, see~\citep{Jolliffe1986} for a comprehensive review of PCA. According to one of its standard formulations, PCA can be interpreted as the search of an orthonormal basis $\bfP$ of $\mathbb{R}^{\dimobs}$ such that all matrices formed by the first $\dimcoef$ columns of $\bfP$ and denoted  $\bfP_{:, 1:\dimcoef}$ ensures
\begin{equation} \label{eq:formationPCA}
	\forall \dimcoef\in\left\{1,\dots,\dimobs\right\}, \quad \bfP_{:, 1:\dimcoef} = \operatornamewithlimits{argmax}_{\bfU \in \stiefel{\dimobs}{\dimcoef} } \bfU\transp \MATobs \MATobs\transp \bfU
\end{equation}
where $\stiefel{\dimobs}{\dimcoef}$ is the Stiefel manifold, i.e., the set of ${\dimobs}\times{\dimcoef}$ orthonormal matrices.

However, Eq.~\eqref{eq:formationPCA} does not provide tools to assert the relevance of the selected principal components in expectation over the data distribution.
To fill this gap, \citet{tipping1999} have shown that PCA can be interpreted as a maximum likelihood estimator of latent factors following the linear model
\begin{equation} \label{eq:intro:modelPPCA}
	\forall \nobs\in\left\{1,\dots,\Nobs\right\}, \quad\Vobs_\nobs = \bfW \bfx_\nobs + \boldsymbol{\epsilon}_\nobs
\end{equation}
where $\Vobs_\nobs$ is the observation vector, $\bfW$ is the matrix of latent factors assumed to be Gaussian, $\bfx_\nobs$ is the associated vector of coefficients and $\epsilon_\nobs$ is an isotropic  Gaussian noise. 
If the coefficients $\bfx_\nobs$ are assumed Gaussian, they can be analytically marginalized out thanks to a natural conjugacy property. The resulting marginalized likelihood function $\mathrm{p}(\Vobs\given \bfW,\, \boldsymbol{\epsilon}_\nobs)$ can be expressed in terms of the empirical covariance matrix $\bfY\transp \bfY$ and the hermitian matrix $\bfW\transp\bfW$.
Although no orthogonality constraint is imposed on the latent factors, the resulting marginal maximum likelihood estimator is precisely provided by the singular value decomposition (SVD) of the noise-corrected observation vector: the SVD produces a set of orthogonal vectors.
The subspace can then be recovered using an expectation-maximization (EM) algorithm. One of the main advantages of this so-called probabilistic PCA (PPCA) lies in its ability to deal with non-conventional datasets. For instance, such an approach allows PCA to be conducted while facing missing data or non linearities \citep{Tipping1999b,tipping1999}. Several works have pursued these seminal contributions, e.g., to investigate these non linearities more deeply \citep{Bolton2003,lawrence2005,lian2009} or the robustness of PPCA with respect to the presence of corrupted data or outliers \citep{Archambeau2008,Schmitt2016}.

Several studies have addressed the issue of determining the relevant latent dimension of the data, $\dimcoef$ here.
The PPCA along with its variational approximation proposed by \cite{bishop1999,export:67241_rd} automatically prunes directions associated with low variances, in the spirit of automatic relevance determination \citep{Mackay95}.
Another strategy considers the latent dimension $\dimcoef$ as a random variable within a hierarchical model of the form $f(\bfW | \dimcoef) f(\dimcoef)$ and uses the SVD decomposition of $\bfW$. However, explicit expressions of the associated estimators are difficult to derive.
To bypass this issue, \cite{minka2000_rd} and \cite{Smidl2007} have proposed Laplace and variational approximations of the resulting posteriors, respectively. Solutions approximated by Monte Carlo sampling are even harder to derive since the size of the parameter space varies with $\dimcoef$. \cite{zhang2004} have proposed to use reversible jump Markov chain Monte Carlo (RJ-MCMC) algorithms \citep{green1995} to build a Markov chain able to explore spaces of varying dimensions. Despite satisfying results, this method is computationally very expensive. 

Bayesian nonparametric (BNP) inference has been a growing topic over the past fifteen years, see for instance the review by \cite{muller2013}. Capitalizing on these recent advances of the BNP literature, this work proposes to use the Indian buffet process (IBP) as a BNP prior to deal with the considered subspace inference problem. More precisely, the basis of the relevant subspace and associated representation coefficients are incorporated into a single Bayesian framework called Bayesian nonparametric principal component analysis or BNP-PCA. A preliminary version of this work was presented at ICASSP 2017~\citep{elvira2017icassp}. Following the approach by \cite{Besson2011}, the prior distribution of the principal components is a uniform distribution over the Stiefel manifold. Then, the IBP permits to model the observations by a combination of a potentially infinite number of latent factors. Inheriting from intrinsic properties of BNP, the IBP naturally penalizes the complexity of the model (i.e., the number $K$ of relevant factors), which is a desired behaviour for dimension reduction. In addition, while the IBP still permits to infer subspaces of potentially infinite dimension, the orthogonality constraint imposed to the latent factors enforces their number $\dimcoef$ to be at most $\dimobs$: orthogonality has some regularization effect as well. The posterior of interest is then sampled using an efficient MCMC algorithm which does not require reversible jumps.

Compared to alternative approaches, in particular those relying on RJ-MCMC sampling, the adopted strategy conveys significant advantages. First, although RJ-MCMC is a powerful and generic tool, its implementation needs the definition of bijections between parameter spaces of different sizes. As a consequence, Jacobian matrices contribute to the probability of jumping between spaces of different dimensions. These Jacobian terms are often both analytically and computationally expensive. Within a BNP framework, there are no such Jacobian terms. Monte Carlo sampling of BNP models implicitly realizes trans-dimensional moves since the IBP prior is a distribution on infinite binary matrices. Combined to the conjugacy properties of the IBP, such a formulation permits more efficient Monte Carlo sampling. Then, the use of the IBP and its induced sparsity alleviates the overestimation of the latent dimension coupled with a subsequent pruning strategy followed by other crude approaches. The proposed model also opens the door to a theoretical analysis of the consistency of estimators. Finally, the method is flexible enough to be coupled with standard machine learning (e.g., classification) and signal processing (e.g., signal decomposition) tasks.




\begin{table}
\begin{center}
\begin{tabular}{cl}
	\toprule
	\headTab{Symbol} & \headTab{Description} \\
	\midrule
	$\Nobs$, $\nobs$ & number of observations, with index \\
	$\dimobs$, $\indobs$ & dimension of observations with index \\
	$\dimcoef$, $\indcoef$ & number of latent factors, with index \\
	$\mathcal{P}(\alpha)$ & Poisson distribution with parameter $\alpha$ \\
	\multirow{2}{*}{$ \stiefel{\dimobs}{\dimcoef} $} & set of $\dimobs\times\dimcoef$ matrices $\MATfac$ such that \\
	& \qquad $\MATfac\transp \MATfac = \indmat_\dimcoef$ \\
	$\mathcal{O}_\dimobs$ & The orthogonal group \\
	$\etr$ & $\exp \trace$ \\
	${}_i\mathrm{F}_j$ & Confluent hypergeometric function \\ 
	${}\gamma(a, b)$ & Lower incomplete Gamma function \\ 
	$\langle \cdot, \cdot \rangle$ & Euclidean scalar product \\
	\bottomrule
	\end{tabular}
\end{center}
	\caption{List of symbols \label{tab:notations}}
\end{table}

This paper is organized as follows. Section~\ref{sec:preliminaries} recalls notions on directional statistics and the IBP.
Section~\ref{sec:model} describes the proposed hierarchical Bayesian model for BNP-PCA.
Section~\ref{sec:algo} describes the MCMC inference scheme.
Section~\ref{sec:estimators} defines several estimators and gathers theoretical results on their properties, in particular their (in-)consistency.
Section~\ref{sec:easyExperiment} illustrates the performance of the proposed method on numerical examples.
Concluding remarks are finally reported in Section~\ref{sec:conclusion}. Note that all notations are gathered in Table~\ref{tab:notations}.

\section{Preliminaries} \label{sec:preliminaries}

\subsection{Distribution on the Stiefel Manifold} \label{subsec:preliminaries:stiefel}
The set of $\dimobs \times \dimcoef$ real matrices $\MATfac$ which verify the relation $\MATfac\transp \MATfac = \indmat_\dimcoef$ is called the Stiefel manifold and is denoted $\stiefel{\dimobs}{\dimcoef}$. Note that when $\dimcoef=\dimobs$, The Stiefel manifold $\stiefel{\dimobs}{\dimobs}$ corresponds to the orthogonal group $\calO_\dimobs$. The Stiefel manifold is compact with finite volume 
\begin{equation} \label{eq:preleminaries:vol_stiefel}
	\mathrm{vol}\left( \stiefel{\dimobs}{\dimcoef} \right) = \frac{2^\dimcoef \pi^{\frac{\dimobs \dimcoef}{2}}}{\pi^{\frac{1}{4} \dimcoef(\dimcoef-1)} \prod_{i=1}^\dimobs \Gamma \left( \frac{\dimobs}{2} - \frac{i-1}{2} \right) } .
\end{equation}
Hence, the uniform distribution $\calU_{ \stiefel{\dimobs}{\dimcoef} }$ on the Stiefel manifold is defined by the density with respect to the Lebesgue measure given by
\begin{equation} \label{eq:preliminaries:pdf_unif_stiefel}
	\prob_{\mathrm{U}} \left( \MATfac \right) = \frac{1}{\mathrm{vol}(\stiefel{\dimobs}{\dimcoef})} \indfunc{ \stiefel{\dimobs}{\dimcoef} }{\MATfac}.
\end{equation}
Over the numerous distributions defined on the Stiefel manifold, two of them play a key role in the proposed Bayesian model, namely the \textit{matrix von Mises-Fisher} and the \textit{matrix Bingham} distributions. Their densities with respect to the Haar measure on the Stiefel Manifold have the following form
\begin{align}
	\prob_{\mathrm{vMF}} \left( \MATfac \given \bfC \right) \; = & \;  _0\mathrm{F}_1^{-1} \left( \emptyset, \frac{\dimobs}{2}, \bfC\transp\bfC \right)   \etr \left( \bfC\transp \MATfac \right) \label{eq:density:vmf} \\
	\prob_{\mathrm{B}} \left( \MATfac \given \bfB \right) \; = & \; _1\mathrm{F}_1^{-1} \left(\frac{\dimobs}{2}, \frac{\dimcoef}{2},\bfB \right)  \etr \left( \MATfac\transp \bfB \MATfac \right) \label{eq:density:bingham}
\end{align}
where $\bfC$ is a $\dimobs\times \dimcoef$ matrix, $\bfB$ is a $\dimobs\times\dimobs$ symmetric matrix and $\etr(\cdot)$ stands for the exponential of the trace of the corresponding matrix. The two special functions $_0\mathrm{F}_1$ and $_0\mathrm{F}_0$ are two confluent hypergeometric functions of matrix arguments \citep{Herz1955}.

\subsection{Nonparametric sparse promoting prior} \label{subsec:preliminaries:ibp}



The Indian buffet process (IBP), introduced by \cite{gg11}, defines a distribution over binary matrices with a fixed number $\Nobs$ of columns but a potentially infinite number of rows denoted by $\dimcoef$.
The IBP can be understood with the following culinary metaphor. Let consider a buffet with an infinite number of available dishes. The first customer chooses $K_1 \sim \mathcal{P}(\alpha)$ dishes. The $\nobs$th customer selects the $k$th dish among those already selected with probability $\frac{m_\indcoef}{\nobs}$ (where $m_\indcoef$ is the number of times dish $k$ has been previously chosen) and tries $K_\nobs \sim \mathcal{P}(\frac{\alpha}{\nobs})$ new dishes. Let $\MATibp$ the binary matrix defined by $z_{k,n} = 1$ if the $\nobs$th customer has chosen the $k$th dish, and zero otherwise.
The probability of any realization of $\MATibp$ is called the \emph{exchangeable feature probability function} by~\cite{broderick2013} and is given by
\begin{equation} \label{eq:intro:pdf_ibp}
	\Prob \big[ \MATibp \given \alpha \big] = \frac{\ibpparam^\dimcoef e^{- \ibpparam\sum_\nobs\frac{1}{\nobs}}}{\prod_{i=1}^{2^\Nobs-1} K_i! } \prod_{\indcoef=1}^\dimcoef \frac{ (\Nobs - m_\indcoef)!(m_\indcoef-1)! }{\Nobs!}
\end{equation}
where $K_i$ denotes the number of times a \emph{history} has appeared: the term \emph{history} refers to a realization of the binary vector of size $\Nobs$ formed by the rows $(z_{k,\cdot})$ of $\MATibp$. Thus, there are $2^\Nobs-1$ possibilities. The IBP can also be interpreted as the asymptotic distribution of a beta Bernoulli process where the beta process has been marginalized out \citep{thibaux2007_rd}. A stick-breaking construction has been also proposed by \cite{TehGorGha2007}. 
We emphasize that the IBP of parameter $\alpha$ is a $\alpha$-sparsity promoting prior since the expected number of non-zero coefficient in $\MATibp$ is of order $\alpha N\log N$.

\section{Bayesian nonparametric principal component analysis (BNP-PCA)} \label{sec:model}

This section introduces a Bayesian method called BNP-PCA for dimension reduction that includes the a priori unknown number of underlying components into the model.
The latent factor model and the associated likelihood function are first introduced in Section~\ref{subsec:model:bayesian}. The prior model is described in Section~\ref{subsec:hierarchicalmodel}. A Monte Carlo-based inference scheme will be proposed in Section~\ref{sec:algo}.

\subsection{Proposed latent factor model} \label{subsec:model:bayesian}

Let $\MATobs = \left[\Vobs_1,\dotsc,\Vobs_\Nobs\right]$ denote the $D\times \Nobs$-matrix of observation vectors $\Vobs_\nobs = [\obs{1,\nobs}, \dotsc, \obs{\dimobs,\nobs}]\transp$. For sake of simplicity but without loss of generality, the sample mean vector $\bar{\Vobs} \triangleq \frac{1}{\Nobs}\sum_{\nobs=1}^{\Nobs} \Vobs_{\nobs}$ is assumed to be zero. Data are supposed to live in an unknown subspace of dimension $\dimcoef \leq \dimobs$. The problem addressed here is thus to identify both the latent subspace and its dimension. To this aim, the observation vectors are assumed to be represented according to the following latent factor model
\begin{equation} \label{eq:model:model}
	\forall \nobs\in\left\{1,\dotsc,\Nobs\right\}, \quad \Vobs_\nobs = \MATfac (\Vibp_\nobs \odot \Vcoef_\nobs) + \Vnoise_\nobs
\end{equation}
where $\MATfac  = \left[\Vfac_1,\dotsc,\Vfac_D\right]$ is an orthonormal base of $\R^{\dimobs}$, i.e., $\MATfac\transp \MATfac = \indmat_{\dimobs}$, $\Vibp_\nobs=\left[z_{1,\nobs}, \ldots, z_{D,\nobs}\right]^T$ is a binary vector, $\Vcoef_\nobs=\left[x_{1,\nobs},\dotsc,x_{D,\nobs}\right]^T$ is a vector of coefficients and $\odot$ denotes the Hadamard (term-wise) product. In Eq.~\eqref{eq:model:model}, the additive term $\Vnoise_\nobs$ can stand for a measurement noise or a modeling error and is assumed to be white and Gaussian with variance $\sigma^2$. It is worth noting that the binary variable $z_{k,\nobs}$ ($k\in \left\{1,\dotsc,D\right\}$) explicitly encodes the activation hence the relevance of the coefficient $x_{k,\nobs}$ and of the corresponding direction $\Vfac_k$ for the latent representation. Thus, the term-wise product vectors $\bfs_\nobs \triangleq \Vibp_\nobs \odot \Vcoef_\nobs$ would be referred to as \emph{factor scores} in the PCA terminology. This is the reason why we call this approach Bayesian nonparametric principal component analysis or BNP-PCA.


The likelihood function is obtained by exploiting the Gaussian property of the additive white noise term. The likelihood of the set of $\Nobs$ observed vectors assumed to be a priori independent can be written as
\begin{equation}
	\begin{split}
	f(\MATobs& \given \MATfac, \MATibp, \MATcoef, \noisevar) =  {} (2 \pi \noisevar)^{- \dimobs \Nobs / 2} \\
	 & \qquad \exp \bigg( -\frac{1}{2\noisevar} \sum_{\nobs = 1}^\Nobs \left\| \Vobs_\nobs - \MATfac (\Vibp_\nobs \odot \Vcoef_\nobs) \right\|_2^2 \bigg),
	\end{split}
\end{equation}
where $\MATibp = \left[\Vibp_1,\dotsc,\Vibp_\Nobs\right]$ is the binary activation matrix, $\MATcoef=\left[\Vcoef_1,\dotsc,\Vcoef_\Nobs\right]$ is the matrix of representation coefficients and $\Vert \cdot \Vert_2$ stands for the $\ell_2$-norm.

\subsection{Prior distributions} \label{subsec:hierarchicalmodel}

The unknown parameters associated with the likelihood function are the orthonormal basis $\MATfac$, the binary matrix $\MATibp$, the coefficients $\MATcoef$ and the noise variance $\noisevar$. Let define the corresponding set of parameters as $\thetabf = (\MATfac, \MATibp, \noisevar)$, leaving $\MATcoef$ apart for future marginalization. \\ \par

\noindent \textbf{Orthonormal basis $\MATfac$.} By definition, $\MATfac$ is an orthonormal basis and belongs to the orthogonal group $\mathcal{O}_{\dimobs} $. Since no information is available a priori about any preferred direction, a uniform distribution on $\mathcal{O}_{\dimobs} $ is chosen as a prior distribution on $\MATfac$ whose probability density function (pdf) with respect to the Lebesgue measure is given by Eq.~\eqref{eq:preliminaries:pdf_unif_stiefel}. \\

\noindent \textbf{Indian buffet process $\MATibp$.} Since the observation vectors are assumed to live in a lower dimensional subspace, most of the factor scores in the vectors $\Vibp_\nobs \odot \Vcoef_\nobs$ are expected to be zero. To reflect this key feature, an IBP prior $\mathrm{IBP}(\alpha)$ is assigned to the binary latent factor activation coefficients, as discussed in Section~\ref{subsec:preliminaries:ibp}. The parameter $\alpha$ controls the underlying sparsity of $\MATibp$. Note that the IBP is a prior over binary matrices with a potentially infinite number of rows $\dimcoef$. However any factor model underlied by a matrix $\MATibp$ with $\dimcoef > \dimobs$ will occur with null probability due to to the orthogonality of $\MATfac$. Our purpose is to combine the flexibility of the IBP prior with the search for an orthogonal projector. 


\noindent \textbf{Coefficients $\MATcoef$.} Independent Gaussian prior distributions are assigned to the individual representation coefficients gathered in the matrix $\MATcoef$. This choice can be easily motivated for large $\Nobs$ by the central limit theorem since these coefficients are expected to result from orthogonal projections of the observed vectors onto the identified basis. Moreover, it has the great advantage of being conjugate to make later marginalization tractable analytically (see next section). To reflect the fact that the relevance of a given direction $\Vfac_k$ is assessed by the ratio between the energy of the corresponding representation coefficients in $x_k$ and the noise variance $\sigma^2$, we follow the recommendation of~\cite{punskaya2002_IEEE_TSP} to define the prior variances of these coefficients as multiples of the noise variance through a Zellner's prior
\begin{equation}
    \label{eq:prior_coefficients}
	\forall \indcoef \in \mathbb{N},\quad \Vcoef_{\indcoef} \given \zellner_\indcoef^2, \noisevar \sim \prod_{\nobs=1}^{\Nobs} \calN(0, \zellner_{\indcoef}^2\noisevar ).
\end{equation}
Along this interpretation, the hyperparameters $\delta_k^2$ would correspond to the ratios between the eigenvalues of a classical PCA and the noise variance.\\

\noindent \textbf{Noise variance $\noisevar$.} A non informative Jeffreys' prior is assigned to $\noisevar$
\begin{equation}
	f(\noisevar) \propto \frac{1}{\noisevar} \indfunc{  \mathbb{R}_{+}}{\noisevar}.
\end{equation}
\noindent \textbf{Hyperparameters.} The set of hyperparameters is gathered in $\boldsymbol{\phi} = \left\{ \boldsymbol{\zellner}, \alpha \right\}$ with $\boldsymbol{\zellner}=\left\{\zellner_1^2,\dotsc, \zellner_\dimcoef^2\right\}$. The IBP parameter $\alpha$ will control the mean number of active latent factors while each hyperparameter $\zellner_\indcoef^2$ scales the power of each component $\Vfac_k$ with respect to the noise variance $\noisevar$. In this work, we propose to include them into the Bayesian model and to jointly estimate them with the parameters of interest. This hierarchical Bayesian approach requires to define priors for these hyperparameters (usually referred to as hyperpriors), which are summarized below. \\

\noindent \textit{Scale parameters $\delta_k^2$.} The powers of relevant components are expected to be at least of the order of magnitude of the noise variance. Thus, the scale parameters $\delta_k^2$ are assumed to be a priori independent and identically distributed according to a conjugate shifted inverse gamma (sIG, see Appendix \ref{app:sIG} for more details) distribution defined over $\R_+$ as in \citep{Godsill2010}
\begin{equation}
\label{eq:sIG}
	\begin{split}
	\prob_{\mathrm{sIG}} \big( \zellner_{\indcoef}^2  \given a_\delta, b_\delta \big) = \frac{ b_\delta^{a_\delta} }{\gamma \left(a_\delta, b_\delta \right)} \qquad \qquad \qquad \qquad \\
	\qquad \quad \bigg(\frac{1}{1 + \zellner_{\indcoef}^2}\bigg)^{a_\delta+1} \exp \left( - \frac{b_\delta}{1 + \zellner_{\indcoef}^2} \right) \indfunc{\R_+}{ \zellner_{\indcoef}^2 }
	\end{split}
\end{equation}
where $\gamma(a,b)$ is the lower incomplete gamma function and $a_{\delta}$ and $b_{\delta}$ are positive hyperparameters chosen to design a vague prior, typically $a=1$ and $b=0.1$.
Note that the specific choice $a_{\delta} = b_{\delta}=0$ would lead to a noninformative Jeffreys prior \citep{punskaya2002_IEEE_TSP}. However, this choice is prohibited here since it would also lead to an improper posterior distribution \citep{Robert2007}.

\noindent \textit{IBP parameter $\alpha$.} Without any prior knowledge regarding this hyperparameter, a Jeffreys prior is assigned to $\ibpparam$. As shown in Appendix \ref{app:ibp_jeffrey}, the corresponding pdf is given by
\begin{equation}
	f(\ibpparam) \propto \frac{1}{\ibpparam} \indfunc{  \mathbb{R}_{+}}{\ibpparam}.
\end{equation}

\section{Inference: MCMC algorithms} \label{sec:algo}

The posterior distribution resulting from the hierarchical Bayesian model for BNP-PCA described in Section~\ref{sec:model} is too complex to derive closed-form expressions of the Bayesian estimators associated with the parameters of interest, namely, the orthonormal matrix $\MATfac$ and the binary matrix $\MATibp$ selecting the relevant components. To overcome this issue, this section introduces a MCMC algorithm to generate samples asymptotically distributed according to the posterior distribution of interest. It also describes a practical way of using these samples to approximate Bayesian estimators. \par

\subsection{Marginalized posterior distribution}

\par A common tool  to reduce the dimension of the  space to be explored while resorting to MCMC consists in marginalizing the full posterior distribution with respect to some parameters. In general, the resulting collapsed sampler exhibits faster convergence and better mixing properties \citep{vanDyk2008}. Here, taking benefit from the conjugacy property induced by the prior in Eq.~\eqref{eq:prior_coefficients}, we propose to marginalize over the coefficients $\MATcoef$ according to the following hierarchical model
\begin{equation}
	f\left( \paramvect,\hypervect  \given \MATobs\right) = \int_{\R^{\dimobs\Nobs} } f\left( \MATobs \given \paramvect,\MATcoef\right) f\left(\paramvect,\MATcoef \given \hypervect \right) f\left(\hypervect\right) \mathrm{d} \MATcoef.
\end{equation}
Calculations detailed in Appendix~\ref{app:marginalize_coeffs} lead to the marginalized posterior distribution
\begin{equation} \label{eq:posterior_marginalized}
	\begin{split}
		\
		f \big( \paramvect, \hypervect& \given \MATobs \big)  \; = {}
		\left(\frac{1}{2\pi\noisevar}\right)^{\frac{\Nobs\dimobs}{2}} \exp \left( -\frac{\mbox{tr}(\MATobs\transp\MATobs)}{2\noisevar}  \right) \\
		\times \; & \prod_{\indcoef=1}^{\dimcoef} \exp \left[ \frac{1}{2 \noisevar } \frac{\zellner_\indcoef^2}{1 + \zellner_\indcoef^2} \sum_{\nobs} \ibp{\indcoef, \nobs} \prodYPscalar \right]  \\
		\times \; & \prod_{\indcoef=1}^{\dimcoef} \left(\frac{1}{1 + \zellner_\indcoef^2}\right)^{a_\delta + \frac{1}{2} \sum_{\nobs} \ibp{\indcoef, \nobs} } \exp \left( - \frac{b_\delta}{1 + \zellner_\indcoef^2} \right) \\
		\times \; & \frac{ \ibpparam^{\dimcoef} }{ \prod_\indcoef \dimcoef_{\nobs}! } e^{-\ibpparam \sum_\nobs \frac{1}{i}} \prod_\indcoef \frac{(\Nobs - m_{\indcoef})! \; (m_\indcoef-1)!}{\Nobs!}
		\\
		\times \; &  \left( \frac{ b_\delta^{a_\delta} }{\gamma \left(a_\delta, b_\delta \right)} \right)^K \left( \noisevar \right)^{-1} \ibpparam^{-1} \indfunc{ \mathbb{U}_{\dimobs} }{ \MATfac } .
	\end{split}
\end{equation}
Note that, since the main objective of this work is to recover a lower dimensional subspace (and not necessarily the representation coefficients of observations on this subspace), this marginalization goes beyond a crude sake of algorithmic convenience. It is also worth noting that it is still possible to marginalize with respect to the scale parameters $\zellnervar_\indcoef$. This finding will be exploited in Section~\ref{subsec:model:inference}.

\subsection{MCMC algorithm} \label{subsec:model:inference}

The proposed MCMC algorithm includes the sampling of $\MATibp$ described in Algo.~\ref{alg:singleton} and is summarized in Algo.~\ref{alg:gibbs}. It implements a Gibbs sampling to generate samples asymptotically distributed according to Eq.~\eqref{eq:posterior_marginalized}. This section derives the conditional distributions associated with the parameters and hyperparameters. \par

\def\mkn{m_{\indcoef}(\nobs)}

\noindent \textbf{Sampling the binary matrix $\MATibp$.} The matrix $\MATibp$ is updated as suggested by \cite{knowles2011_rd}, see Algo.~\ref{alg:singleton}.
Let $\mkn = \sum_{i\neq n} z_{k,i}$ the number of observations different from $\nobs$ which actually use the direction $\Vfac_\indcoef$, \textit{i.e.} verifying $\ibp{\indcoef, i}=1$ for $i \neq \nobs$.
Directions for which $\mkn = 0$ are called \emph{singletons} and the corresponding indices are gathered in a set denoted by $\calJ_n$. Conversely, directions for which $\mkn > 0$ are referred to as \emph{non-singletons} and the set of corresponding indices is denoted by $\calI_n$. Note that $\forall n,\: \calI_n \cup \calJ_n = \{1, \dots, \dimcoef\}$.
First, non-singletons are updated through a Gibbs sampling step where $\zellnervar_\indcoef$ can be marginalized out. One has
\begin{equation} \label{eq:model:posterior:ibp:normal}
	\begin{split}
		&\frac{\Prob \left( \ibpkn=1 \given \MATobs, \MATfac, \noisevar \right) }{\Prob \left( \ibpkn=0 \given \MATobs, \MATfac, \noisevar \right)} = \\
		\quad &\frac{ \mkn }{ \Nobs - 1 - \mkn }
		\exp \left(\frac{1}{2\noisevar} \scalarProdSquare{\Vfac_\indcoef}{\Vobs_\nobs}\right) \times \\
		& \frac{ \gamma\left(a + 1, b + \frac{1}{2\noisevar} \scalarProdSquare{\Vfac_\indcoef}{\Vobs_\nobs} \right) }{ \gamma(a, b) } \; \frac{b^a}{ \left(b + \frac{1}{2\noisevar} \scalarProdSquare{\Vfac_\indcoef}{\Vobs_\nobs} \right)^{a + 1} }
	\end{split}
\end{equation}
where
\begin{align}
	a {}={} & a_\delta + \sum_{i =1, i\neq \nobs}^\Nobs  \ibp {\indcoef, i} \\
	b {}={} & b_\delta + \frac{1}{2\noisevar} \sum_{i =1, i\neq \nobs}^\Nobs  \ibp {\indcoef, i} \scalarProdSquare{\Vfac_\indcoef}{\Vobs_i} .
\end{align}

\begin{figure}
	\centering
	\includegraphics[width=\columnwidth]{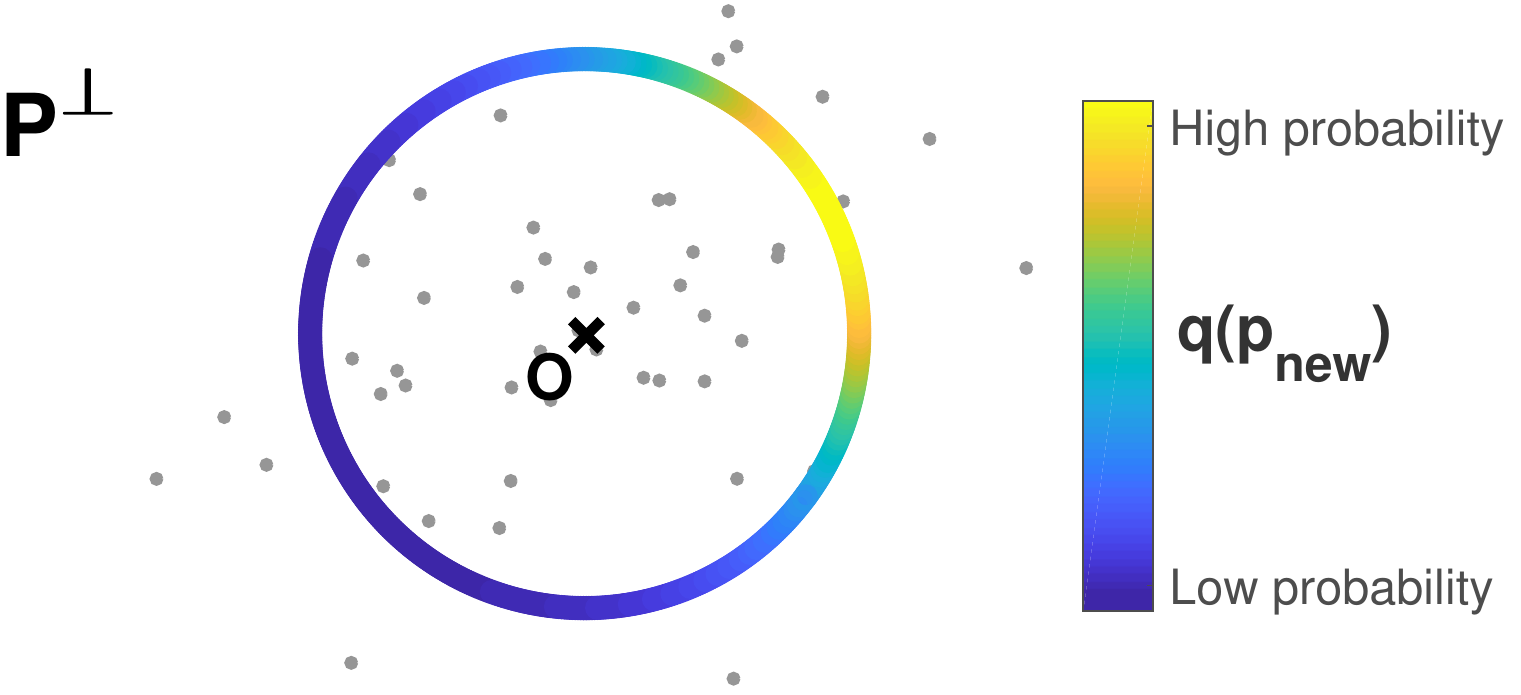}
	\caption{An example of the proposition of new directions when $\MATfac^\perp$ is $2$ dimensional and $\kappa^\star=1$.
	Gray dots are observations projected on $\MATfac^\perp$.
	The colored circle is the pdf of the proposal distribution when $\kappa^\star=1$.}
	\label{fig:ajoutAtome}
\end{figure}

A Metropolis Hastings step is used to update singletons. Let $\kappa = \mathrm{card}\big(\calJ_n\big)$ be the number of singletons, $\MATfac_{\calI_n}$ and $\MATfac_{\calJ_n}\triangleq [\widetilde{\Vfac}_1, \dots \widetilde{\Vfac}_\kappa]$ be the sub-matrices of $\MATfac$ with indices in $\calI_n$ and $\calJ_n$, respectively. The move goes from a current state $\mathbf{s}=\left\{ \kappa, \MATfac_{\calJ_n} \right\}$ to a new state $\mathbf{s}^{\star}=\left\{ \kappa^{\star}, \MATfac_{\calJ_n^{\star}}^{\star} \right\}$.
The proposal distribution in the Metropolis-Hastings step is chosen according to the conditional model
\begin{equation}
	\label{eq:model:full_proposal}
	\mathrm{q}\left( \kappa^{\star}, \MATfac_{\calJ_n^{\star}}^{\star} | \kappa, \MATfac_{\calJ_n}, \MATfac_{\calI_n} \right)
	= \mathrm{q}\left( \kappa^{\star} | \MATfac_{\calI_n} \right) \mathrm{q}\left( \MATfac_{\calJ_n^{\star}}^{\star} | \kappa^{\star}, \MATfac_{\calI_n} \right) .
\end{equation}

\begin{algorithm}
	\caption{Detailed procedure to sample $\MATibp$}
	\label{alg:singleton}
	\KwIn{$\MATobs, \MATibp^{(t-1)}, \MATfac^{(t-1)}, \noisestd^{2\,(t-1)}, \delta_\indcoef^{2\,(t-1)}$}
	\BlankLine
	Let $\MATfac^{(t-\frac{1}{2})} = \MATfac^{(t-1)}$ \;
	\For{$\nobs \gets 1$ \KwTo $\Nobs$} {
		\tcp{Identify shared directions and singletons}
		\For{$\indcoef \gets 1$ \KwTo $\dimcoef$} {
			Compute $\mkn = \sum_{l\neq \nobs} \ibp{\indcoef, l}^{(t-1)}$\;
		}
		Let $\calI_n\; \triangleq \big\{\indcoef, \; \mkn > 0 \big\}$  \;
		Let $\calJ_n \triangleq \big\{\indcoef, \; \mkn = 0 \big\}$  \;
		\tcp{Sample shared directions}
		\ForEach{$\indcoef \text{ in } \calI$} {
			Sample $\ibpkn^{(t)}$ according to Eq.~\eqref{eq:model:posterior:ibp:normal} \;
		}

		\tcp{Define set of singletons}
		Let $\kappa\; \triangleq \mathrm{card} \big(\calJ_n \big)$  \;
		Let $\MATfac_{\calI_n} \triangleq \big[ \Vfac_\indcoef, \indcoef \in \calI_n \big]$. \;
		\tcp{Sample new number of singletons}
		Sample $\kappa^{\star}$ according to Eq.~\eqref{eq:model:proposalKappa} \;
		\tcp{Sample iteratively new directions}
		Let $\MATfac_{\calJ_n^\star}^\star = [\,]$ \;
		\For{$\indcoef \gets 1$ \KwTo $\kappa^{\star}$} {
			Let $\bfN$ an orthonormal basis of $\big[\MATfac_{\calI_n}, \MATfac_{\calJ_n^\star}^\star\big]^{\perp}$\;
			Let $\bfv$ the first eigenvector of $\bfN\transp \MATobs \MATobs\transp \bfN$ and $\lambda$ its associated eigenvalue \;
			Sample $\Vfac_\indcoef^{\star} \sim \mathrm{vMF} (\bfv, \lambda)$ \;
			Update $\MATfac_{\calJ_n^\star}^\star = \Big[ \MATfac_{\calJ_n^\star}^\star, \Vfac_\indcoef^{\star} \big]$ \;
		}
		\tcp{Metropolis Hasting step}
		Compute $u_{\mathbf{s} \rightarrow \mathbf{s}^{\star}}$ according to Eq.~\eqref{eq:model:posterior:acceptation_singleton} \;
		Sample $u \sim \mathcal{U}([0, 1])$ \;
		\If{$ u \leq u_{\mathbf{s} \rightarrow \mathbf{s}^{\star}} $} {
			Set $\mathbf{s} = \mathbf{s}^{\star}$ and update $\MATfac^{(t-\frac{1}{2})}$ \;
			Update $\dimcoef = \dimcoef - \kappa + \kappa^{\star}$ \;
		}
	}
	\BlankLine
	\KwOut{$\MATibp^{(t)}, \MATfac^{(t-\frac{1}{2})}$.} 
\end{algorithm}

Note that the proposal distribution Eq.\eqref{eq:model:full_proposal} is conditioned to $\MATfac_{\calI_n}$.
This choice is legit since the goal is to sample the conditional distribution $f(\MATibp, \MATfac_{\calJ_n} | \MATobs, \MATfac_{\calI_n}, \noisevar)$.
Close to the structure of the IBP, we propose to use for $\mathrm{q}\left( \kappa^{\star} | \MATfac_{\calI_n} \right)$ a Poisson distribution $\mathcal{P}(\ibpparam)$ combined with a mass $\mathrm{card}\big(\calI_n\big) / \dimobs$ on $0$:
\begin{equation} \label{eq:model:proposalKappa}
	 \mathrm{q}\left( \kappa^{\star} | \MATfac_{\calI_n} \right) = \frac{\mathrm{card}\big(\calI_n\big)}{\dimobs} \delta_0(\kappa) + \left(1 - \frac{\mathrm{card}\big(\calI_n\big) }{\dimobs}\right) \mathcal{P}(\alpha)
\end{equation}
Recall that $\mathrm{card}\big(\calI_n\big)$ is the number of coefficients $z_{k,n}=1$ of the $n$th column of $\MATibp$ that are not singletons (singletons$\Leftrightarrow z_{k,n}=1$ \& $\forall i\neq n$, $z_{k,i}=0$). Once $\kappa^{\star}$ has been chosen, a new matrix $\MATibp^{\star}$ is formed by concatenating columns with indices in $\calI_n$ and $\kappa^\star$ rows with zeros everywhere except ones at the $\nobs^{\mathrm{th}}$ position (or column).

For $\MATfac_{\calJ_n}$, a von Mises-Fisher distribution $\mathrm{vMF}(\bfC)$, see Section~\ref{subsec:preliminaries:stiefel}, is chosen as a proposal. The columns of $\bfC$ are built from the $\kappa$ first eigenvectors of the projection of $\MATobs\MATobs^T$ on the orthogonal of $\MATfac_{\calI_n}$, i.e. the span of singletons and unused directions.
The columns of $\bfC$ are then multiplied by their corresponding eigenvalues. Figure~\ref{fig:ajoutAtome} illustrates the procedure to add one new direction, $\kappa^\star = 1$, on a simple example in dimension 2.

\noindent The move $\mathbf{s} \rightarrow \mathbf{s}^{\star}$ is then accepted with probability
\begin{equation} \label{eq:model:posterior:acceptation_singleton}
	u_{\mathbf{s} \rightarrow \mathbf{s}^{\star}}
	= \frac{ f \left( \MATobs \given \MATfac_{\calI_n}, \MATfac_{\calJ_n}^\star, \MATibp^\star, \noisevar \right) }{ f \left( \MATobs \given \MATfac_{\calI_n}, \MATfac_{\calJ_n}, \MATibp, \noisevar \right) }
	\frac{ \prob \left( \mathbf{s}^{\star} \right) \qprob(\mathbf{s} \given \mathbf{s}^{\star}, \MATfac_{\calI_n})}{\prob \left( \mathbf{s} \right) \qprob(\mathbf{s}^{\star} \given \mathbf{s}, \MATfac_{\calI_n})}
\end{equation}
The full procedure is summarized in Algo.~\ref{alg:gibbs}.

\begin{algorithm}[th]
	\caption{Gibbs sampler} 
	\label{alg:gibbs}
	\KwIn{$\MATobs$, $n_{\mathrm{mc}}$}
	\BlankLine
	\For{$t \gets 1$ \KwTo $n_{\mathrm{mc}}$} {
			\tcp{Update directions and handle singletons}
			Sample $\MATibp^{(t)}$ and $\MATfac^{(t-\frac{1}{2})}$ as described in Alg.~\ref{alg:singleton} \;

			\tcp{Update activated directions and weights.}
			\For{$\indcoef \gets 1$ \KwTo $\dimcoef$} {
				Compute $\bfN_{K\backslash \indcoef}$, a basis of $\MATfac_{\backslash \indcoef}^{\perp\, (t-\frac{1}{2})}$ \;
				Sample $\bfv_\indcoef$ according to Eq.~\eqref{eq:posterior:Vk} \;
				Set $\Vfac_\indcoef^{(t)} = \bfN_{K\backslash \indcoef} \bfv_\indcoef$ \;
				Sample $\zellner_\indcoef^{2\,(t)}$ according to Eq.~\eqref{eq:posterior:deltaX} \;
			}
		\tcp{Update hyperparameters.}
		Sample $\noisestd^{2\,(t)}$ according to Eq.~\eqref{eq:model:posterior:noisevar} \;
		Sample $\ibpparam^{(t)}$ according to Eq.~\eqref{eq:model:posterior:alpha} \;

	}
	\BlankLine
	\KwOut{A collection of samples $\left\{ \MATfac^{(t)}, \MATibp^{(t)}, \zellner_\indcoef^{2\,(t)}, \noisestd^{2\,(t)}, \ibpparam^{(t)} \right\}_{t=n_{\mathrm{burn}}+1}^{n_{mc}}$ asymptotically distributed according to Eq.~\eqref{eq:posterior_marginalized}.}
\end{algorithm}

\noindent \textbf{Sampling the orthonormal basis $\MATfac$.}

\noindent Let $\mathcal{A} \subset \left\{1,\dotsc,D\right\} $ denote the set of $K$ indices corresponding to the active directions in $\MATfac$, i.e., the $K$ columns of $\MATfac$ actually used by at least one observed vector: $\forall k\leq\dimcoef$, $\exists n$ s.t. $z_{k,n}=1$. Matrix $\MATfac$ can be split into $2$ parts $\MATfac = [\MATfac_{\mathcal{A}},\MATfac_{\bar{\mathcal{A}}}]$. The matrix $\MATfac_{\mathcal{A}}$ features the $K$ active directions and $\MATfac_{\bar{\mathcal{A}}}$ the $(D-K)$ unused components. Let $\MATfac_{\mathcal{A} \backslash \indcoef}$ denote the matrix obtained by removing the column $\Vfac_k$ from $\MATfac_{\mathcal{A}}$ and $\bfN_{\mathcal{A} \backslash\indcoef}$ a matrix whose $(D-K+1)$ columns form an orthonormal basis for the orthogonal of $\MATfac_{\mathcal{A}\backslash\indcoef}$.
Since $\Vfac_\indcoef\in\MATfac_{\mathcal{A} \backslash \indcoef}^\perp$ it can be written as $\Vfac_\indcoef = \bfN_{\mathcal{A} \backslash \indcoef} \bfv_\indcoef$.  Since the prior distribution of $\MATfac$ is uniform on the orthogonal group $\mathcal{O}_{\dimobs}$, $\bfv_k$ is uniform on the $(\dimobs - \dimcoef+1)$-dimensional unit sphere \citep{h07_rd}. By marginalizing ${\MATfac}_{\bar{\mathcal{A}}}$, one obtains
\begin{equation} \label{eq:posterior:Vk}
	\begin{split}
		f(\bfv_{\indcoef} \given \MATobs, \MATfac_{\calA\backslash \indcoef}, \MATibp, \zellnervar_{\indcoef}, \noisevar) \propto \qquad \qquad \qquad \qquad \qquad \qquad \\
		\displaystyle{\exp\left(\frac{1}{2\noisevar} \frac{\zellnervar_\indcoef}{1+\zellnervar_\indcoef} \bfv_{\indcoef}\transp \bfN_{\calA \backslash \indcoef}\transp \left( \sum_{\nobs=1}^\Nobs \ibpkn \Vobs_\nobs\Vobs_\nobs\transp \right) \bfN_{\calA \backslash \indcoef} \bfv_\indcoef  \right)}
	\end{split}
\end{equation}
which is a Bingham distribution on the $(D-K+1)$-unit sphere, see Section~\ref{subsec:preliminaries:stiefel}. As a consequence,
\begin{equation} \label{eq:posterior:Pk}
	\begin{split}
		f(\Vfac_\indcoef  \given \MATobs, \MATfac_{\calA\backslash \indcoef}, \MATibp, \zellnervar_{\indcoef}, \noisevar) \propto \qquad \qquad \qquad \qquad \qquad \qquad \\
		\displaystyle{\exp\left(\frac{1}{2\noisevar} \frac{\zellnervar_\indcoef}{1+\zellnervar_\indcoef} \Vfac_\indcoef\transp \left( \sum_{\nobs=1}^\Nobs \ibpkn \Vobs_\nobs\Vobs_\nobs\transp \right)\Vfac_\indcoef  \right)}
	\end{split}
\end{equation}


\noindent \textbf{Sampling the scale parameters $\zellnervar_{\indcoef}$.} The posterior distribution of $\zellner_\indcoef^2$ for all $\indcoef$, can be rewritten as
\begin{equation} \label{eq:posterior:deltaX}
	\begin{split}
	f \left( \zellner_\indcoef^2  \given \MATfac, \MATibp, \noisevar \right) {} \propto
	\left(\frac{1}{1 + \zellner_\indcoef^2}\right)^{a_{\zellner} + \frac{1}{2} \sum_{\nobs} \ibp{\indcoef, \nobs} + 1} \qquad \\
	\quad \exp \left[ - \frac{1}{1 + \zellnervar_{\indcoef}} \bigg( b_\zellner + \frac{1}{2 \noisevar } \sum_{\indcoef, \nobs} \ibp{\indcoef, \nobs} \scalarProdSquare{\Vfac_{\indcoef}}{\Vobs_{\nobs} } \bigg) \right] .
	\end{split}
\end{equation}
which is a shifted Inverse Gamma distribution.

\noindent \textbf{Sampling the noise variance $\noisevar$.} By looking carefully at \eqref{eq:posterior_marginalized}, one obtains 
\begin{equation} \label{eq:model:posterior:noisevar}
	\begin{split}
		\noisevar \given \MATobs, \MATibp, \MATfac, \boldsymbol{\zellner} \sim \mathcal{IG} \bigg( \frac{\Nobs \dimobs}{2}, \qquad \qquad \qquad \\
		{} \frac{1}{2} \trace\left(\MATobs\MATobs\transp\right) - \sum_{\indcoef, \nobs} \frac{1}{2}\frac{\zellner_\indcoef^2}{1 + \zellner_\indcoef^2} \ibpkn \scalarProdSquare{\Vobs_\nobs}{\Vfac_\indcoef} \bigg) .
	\end{split}
\end{equation}

~\\
\noindent \textbf{Sampling the IBP parameter $\ibpparam$.} The conditional posterior distribution of $\ibpparam$ is gamma distributed
\begin{equation} \label{eq:model:posterior:alpha}
	\alpha \given \MATobs, \MATibp {} \sim
	\mathcal{G} \left( \dimcoef, \sum_{n=1}^N \frac{1}{n} \right) .
\end{equation}
Algo.~\ref{alg:gibbs} describes the full sampling procedure.




\section{Estimators: theoretical properties} \label{sec:estimators}


Since one motivation of the proposed BNP-PCA approach is its expected ability to identify a relevant number of degrees of freedom of the proposed model, this section focuses on this aspect. Section~\ref{subsec:estimators:behaviour_pk} derives theoretical results concerning the marginal maximum a posteriori (MAP) estimator of $\dimcoef$ associated with the proposed IBP-based model. In particular, Theorem~\ref{th:inconsistance} apparently brings some bad news by showing that this estimator is not consistent when the parameter $\alpha$ of the IBP is fixed. Similar results have been reported by \cite{chen2016} on an empirical basis only. Note that our approach considers $\alpha$ as an unknown parameter as well, which may explain the good behaviour observed experimentally in Section~\ref{sec:easyExperiment}. Section~\ref{subsec:selecting_true_K} proposes an efficient way to select the right number of components based on simple statistical tests. Section~\ref{subsec:model:estimators} deals with estimators of other parameters.

\subsection{Posterior distribution of the subspace dimension} \label{subsec:estimators:behaviour_pk}

The consistency of Dirichlet process mixture models (DPMMs) for Bayesian density estimation has been widely studied, see \cite{ghosal2009} and references therein. For instance, posterior consistency of such DPMMs with a normal kernel has been obtained by \cite{ghosal1999}.
While such results tend to motivate the use of nonparametric priors, a certain care should be paid regarding the behaviour of any posterior distribution. For instance, \cite{mccullagh2008} have provided both experimental and analytical results about the ability of DPMMs to identify and separate two clusters.
More recently, \cite{miller2013,miller2014} have shown that the posterior distribution of the number of clusters of DPMMs and Pitman-Yor process mixture models are not consistent. When the number of observations tends to infinity, the marginal posterior does not concentrate around any particular value, despite the existence of concentration rates. Fewer results are available when an IBP is used, see for instance \cite{chen2016} where posterior contraction rates are established for phylogenetic models.

The following theorem shows that the marginal MAP estimator of the number of components $K$ is not consistent when conditioned upon (fixed) $\alpha$.
\begin{theorem} \label{th:inconsistance}
	Let $\MATobs_\Nobs=\left[\Vobs_1, \ldots, \Vobs_\Nobs\right]$ denote a matrix of $\Nobs$  $\dimobs$-dimensional observations. Let $\dimcoef_\Nobs$ denotes the random variable associated with the latent subspace dimension of the model described in Section~\ref{sec:model}. Then, the two following assertions
	\begin{align}
		\forall \indcoef < \dimobs \quad & \underset{\Nobs\rightarrow\infty}{\lim\sup} \;\; \mathrm{P}\big[ \dimcoef_N = k \mid \MATobs_\Nobs, \alpha \big] &&{}<{} 1 \label{eq:th:consistance:estimator:K_k} \\
		&\underset{\Nobs\rightarrow\infty}{\lim\sup} \;\; \mathrm{P}\big[ \dimcoef_\Nobs = \dimobs \mid \MATobs_\Nobs, \alpha \big] &&{}>{} 0 \label{eq:th:consistance:estimator:K_D}
	\end{align}
	are true.
\end{theorem}
\begin{proof}
See Appendix~\ref{app:inconsistency}.
\end{proof}
As discussed in the proof, Eq.~\eqref{eq:th:consistance:estimator:K_k} can be extended to a wider range of models, while Eq.~\eqref{eq:th:consistance:estimator:K_D} results from the orthogonality constraint. Up to our knowledge, no similar results have been derived for the IBP. We emphasize that Theorem~\ref{th:inconsistance} does not claim that the marginal MAP estimator of the subspace dimension defined as
\begin{equation}
\label{eq:K_mMAP_alpha}
  \widehat{K}_{\mathrm{mMAP},\alpha} = \operatornamewithlimits{argmax}_{k \in \{0,\dotsc,D\}} \mathrm{P}\big[ \dimcoef = k \mid \MATobs_\Nobs, \alpha \big]
\end{equation}
is biased or irrelevant. However, a corollary of Eq. \eqref{eq:th:consistance:estimator:K_k} is that this estimator is not consistent. This can be explained by a certain leakage of the whole mass towards the probability of having $\dimcoef = \dimobs$, as shown by Eq. \eqref{eq:th:consistance:estimator:K_D}.
To overcome this issue, instead of resorting to the conventional marginal MAP estimator of the dimension, an alternative strategy will be proposed in Section \ref{subsec:selecting_true_K} to identify the dimension of the relevant subspace.

By considering an additional hypothesis on the distribution of the measurements $\MATobs_{\Nobs}$, the following theorem states an interesting result.
\begin{theorem} \label{th:inconsistance_severe}
	Let $\Vobs_1, \ldots, \Vobs_\Nobs$ be $\Nobs$  $\dimobs$-dimensional observations independently and identically distributed according to a centered Gaussian distribution of common variance $\sigma^2_{\bfy}$. Then
	\begin{equation}
		\mathrm{P} \big( K_\Nobs = 0 | \MATobs_\Nobs, \ibpparam, \sigma^2_{\bfy} \big) \overset{\mathrm{a.s.}}{\underset{\Nobs \rightarrow + \infty}{\longrightarrow}} 0 .
	\end{equation}
\end{theorem}
\begin{proof}
  See Appendix~\ref{app:severe_inconsistency}.
\end{proof}
Two distinct interpretations of this theorem can be proposed. Indeed the Gaussian assumption is used twice in this case: both the data and the noise are Gaussian. On one hand, if Gaussian measurements are interpreted as noise, i.e., $\Vobs_\nobs = \boldsymbol{\epsilon}_\nobs$ and $\sigma^2_{\bfy}=\sigma^2$ in the proposed latent factor model \eqref{eq:intro:modelPPCA}, the expected dimension of the latent subspace should be $0$. Theorem~\ref{th:inconsistance_severe} states that this will almost surely not be the case, so that
$\widehat{\dimcoef}_\Nobs$ is inconsistent.
On the other hand, the same Theorem~\ref{th:inconsistance_severe} can be positively interpreted since one would rather expect to find $\widehat{\dimcoef}_\Nobs = \dimobs$ since white Gaussian noise spreads its energy equally in every direction. With respect to this second interpretation,
$\widehat{\dimcoef}_\Nobs$ may be considered as consistent.


In the present approach, we consider that a latent subspace is meaningful as soon as it permits to distinguish a signal from white Gaussian noise: we stick to the first interpretation of Theorem~\ref{th:inconsistance_severe} and consider that $\widehat{\dimcoef}_\Nobs$ is inconsistent.
Finally, we emphasize that the two theorems above are related to posterior estimators of $\dimcoef$ conditioned upon $\alpha$ and possibly $\sigma^2$. A posterior estimator $\widehat{K}_{\mathrm{mMAP}}$ will be defined later by Eq.\eqref{eq:K_mMAP} where parameters $\alpha$ and $\sigma^2$ are marginalized. Experiments conducted in Section \ref{sec:easyExperiment} will show that this $\widehat{K}_{\mathrm{mMAP}}$ seems to be asymptotically consistent.

\subsection{Selecting the number of components} \label{subsec:selecting_true_K}

As emphasized in Section~\ref{subsec:estimators:behaviour_pk}, the posterior probabilities $\mathrm{P}\big[\dimcoef \given \MATobs,\ibpparam\big]$ may not to be sufficient to properly derive reliable estimates of the subspace dimension and select the number of relevant directions. However, the proposed BNP-PCA considers the IBP parameter $\alpha$ as unknown. Then one can define the marginalized MAP estimate
\begin{equation}
\label{eq:K_mMAP}
  \widehat{K}_{\mathrm{mMAP}} = \operatornamewithlimits{argmax}_{k \in \{0,\dotsc, \dimobs\}} \mathrm{P}\big[ \dimcoef = k \mid \MATobs \big].
\end{equation}
The numerical study of Section~\ref{sec:easyExperiment} will show that it seems to be consistent contrary to $\widehat{K}_{\mathrm{mMAP},\alpha}$. As a consequence, as soon as sufficient amount of data is available, one may use $\widehat{K}_{\mathrm{mMAP}}$ for model selection.


Another possibility, with theoretical guarantees, is to take advantage of the posterior distribution of the principal components $\MATfac$ and to use statistical tests. In accordance with the notations introduced in Section \ref{subsec:model:inference}, let $\mathcal{A} \subset \left\{1,\dotsc,D\right\} $ denote the set of $K$ indices corresponding to the estimated active directions in $\MATfac$.
Elaborating on \eqref{eq:posterior:Pk}, the posterior distribution of $\Vfac_\indcoef$, $\forall \indcoef \in \bar{\mathcal{A}}$, can be expressed thanks to a $(D-K)$-dimensional unit-norm random vector $\boldsymbol{w}_{\indcoef} = \bfN_{\mathcal{A}}^T \Vfac_{\indcoef}  $ whose distribution is given by
\begin{equation}
\label{eq:distribution_vk}
  f\left(\boldsymbol{w}_{\indcoef} \given \MATobs_{\Nobs}, \MATfac_{\mathcal{A}}, \MATibp, \zellnervar_{\indcoef}, \noisevar\right)  \propto \exp\left(\boldsymbol{w}_{\indcoef}^T \boldsymbol{\Lambda}_{\indcoef,{\Nobs}} \boldsymbol{w}_{\indcoef} \right)
\end{equation}
where
\begin{equation}
\label{eq:distribution_vk_param}
  \boldsymbol{\Lambda}_{\indcoef,{\Nobs}}= \gamma_{\indcoef}  \sum_{n=1}^N\bfN_{\mathcal{A}}^T \bfy_n \bfy_n^T \bfN_{\mathcal{A}}
\end{equation}
where $\bfN_{\mathcal{A}}$ is a $D\times (D-K)$ orthogonal matrix which spans the null space of $\MATfac_{\mathcal{A}}$; $\gamma_{\indcoef} $ depends on $\noisevar$ and $\zellnervar_\indcoef$. Interestingly, if $\MATfac_{\mathcal{A}}$ correctly identifies the unknown signal subspace of dimension $K$, any component $\Vfac_\ell$, $\ell\in \bar{\mathcal{A}}$ is actually a non-relevant direction. According to the latent factor model \eqref{eq:intro:modelPPCA}, the projected vectors $\bfN_{\mathcal{A}}^T \bfy_n$ ($n=1,\dotsc,\Nobs$) in \eqref{eq:distribution_vk_param} should reduce to white Gaussian noises so that
\begin{equation}
\label{eq:projected_vectors}
  \operatornamewithlimits{lim}_{\Nobs  \rightarrow + \infty} \frac{1}{\Nobs}\boldsymbol{\Lambda}_{\indcoef,{\Nobs}} = \gamma_{\indcoef} \sigma^2 \indmat_{D-K}
\end{equation}
where $\indmat_{D-K}$ is the $(D-K)$ identity matrix. This means that the posterior distribution \eqref{eq:distribution_vk} of the $\boldsymbol{w}_{\ell}$ tends to be uniform over the $(D-K)$-dimensional sphere.
Let $L=D-K$ and $\boldsymbol{\mathcal{W}}_{\bar{\mathcal{A}}}$ the $L\times L$ orthogonal matrix whose columns are the vectors $\left\{\boldsymbol{w}_\ell\right\}_{\ell \in \bar{\mathcal{A}}}$. One could think of building tests of goodness-of-fit able to identify the maximum dimension $L=D-K\in\{0,\dotsc,D\}$ for which $\boldsymbol{\mathcal{W}}_{\bar{\mathcal{A}}}$ remains uniformly distributed over the orthogonal group $\mathcal{O}_{L}$.
However, since $\boldsymbol{\mathcal{W}}_{\bar{\mathcal{A}}}$ lives in a possibly high dimensional space, this testing procedure would be inefficient to provide a reliable decision rule.
As an alternative, we propose to conduct a statistical tests on the set of the following $L=D-K$ absolute scalar products
\begin{equation}
\label{eq:omega_absolute}
	\omega_\ell \triangleq |\boldsymbol{w}_{\ell}^T \bfu_\ell|, \quad \ell \in \bar{\mathcal{A}},
\end{equation}
where the $\left\{\bfu_\ell\right\}_{\ell \in\bar{\mathcal{A}}}$  is a set of $L$ arbitrary $L$-dimensional unit-norm vectors, for instance uniformly distributed on the sphere.
Indeed, if $\boldsymbol{\mathcal{W}}_{\bar{\mathcal{A}}}$ is uniformly distributed over the orthogonal group $\mathcal{O}_{L}$, the distribution of the $L$-dimensional random vector
$\boldsymbol{\omega}_{\bar{\mathcal{A}}}$ whose components are given by \eqref{eq:omega_absolute} can be easily derived as stated by the following theorem.

\begin{theorem} \label{th:distribution_scalar_product}
	Let $K \in \left\{0, \ldots,\dimobs \right\}$, $\boldsymbol{\mathcal{W}}=\left[\boldsymbol{w}_1,\dotsc,\boldsymbol{w}_{D-K}\right]^T$ be a random matrix uniformly distributed on the orthogonal group $\mathcal{O}_{D-K}$,
	and $\bfu_{1}, \ldots, \bfu_{\dimobs-K}$ be  $L=(\dimobs - K)$ arbitray unit-norm $L$-dimensional vectors.
	Let $\boldsymbol{\omega} = \left[ \omega_{1}, \ldots, \omega_{L} \right]^T$ such that $\omega_{\ell} \triangleq |\boldsymbol{w}_\ell\transp \bfu_\ell |$.
	Then, the components of $\boldsymbol{\omega}$ are identically distributed and the cumulative distribution (cdf) of any component $\omega_\ell$ is given by
	\begin{equation} \label{eq:cdf_scalar_product}
	\begin{split}
		\mathrm{P}\left( \omega_\ell \leq \lambda \right)
		&{}={} \frac{\mathrm{vol}\left(\mathcal{O}_{L-2}\right)}{\mathrm{vol}\left(\mathcal{O}_{L-1}\right)} \; 2\int_0^{\lambda} \left( 1 - z^2 \right)^{(L-3) / 2} \mathrm{d}z \\
		&{}={} 2\lambda_l \frac{\mathrm{vol}\left(\mathcal{O}_{L-2}\right)}{\mathrm{vol}\left(\mathcal{O}_{L-1}\right)} \;_2F_1\left( \frac{1}{2}, -\frac{L-3}{2} ; \frac{3}{2} ; \lambda^2 \right).
	\end{split}
	\end{equation}
\end{theorem}
\begin{proof}
  See Appendix~\ref{app:expectation_scalar_product}.
  \end{proof}

Note that the $\omega_\ell$ can be interpreted as generalized cosines in dimension $L=D-K$. The distribution Eq.~\eqref{eq:cdf_scalar_product} depends on the difference $\dimobs-K$ only. Fig.~\ref{fig:produit_scalaire} shows the empirical and theoretical pdf's associated with the cdf \eqref{eq:cdf_scalar_product} for various values of $D-K$.

\begin{figure}
	\includegraphics[width=\columnwidth]{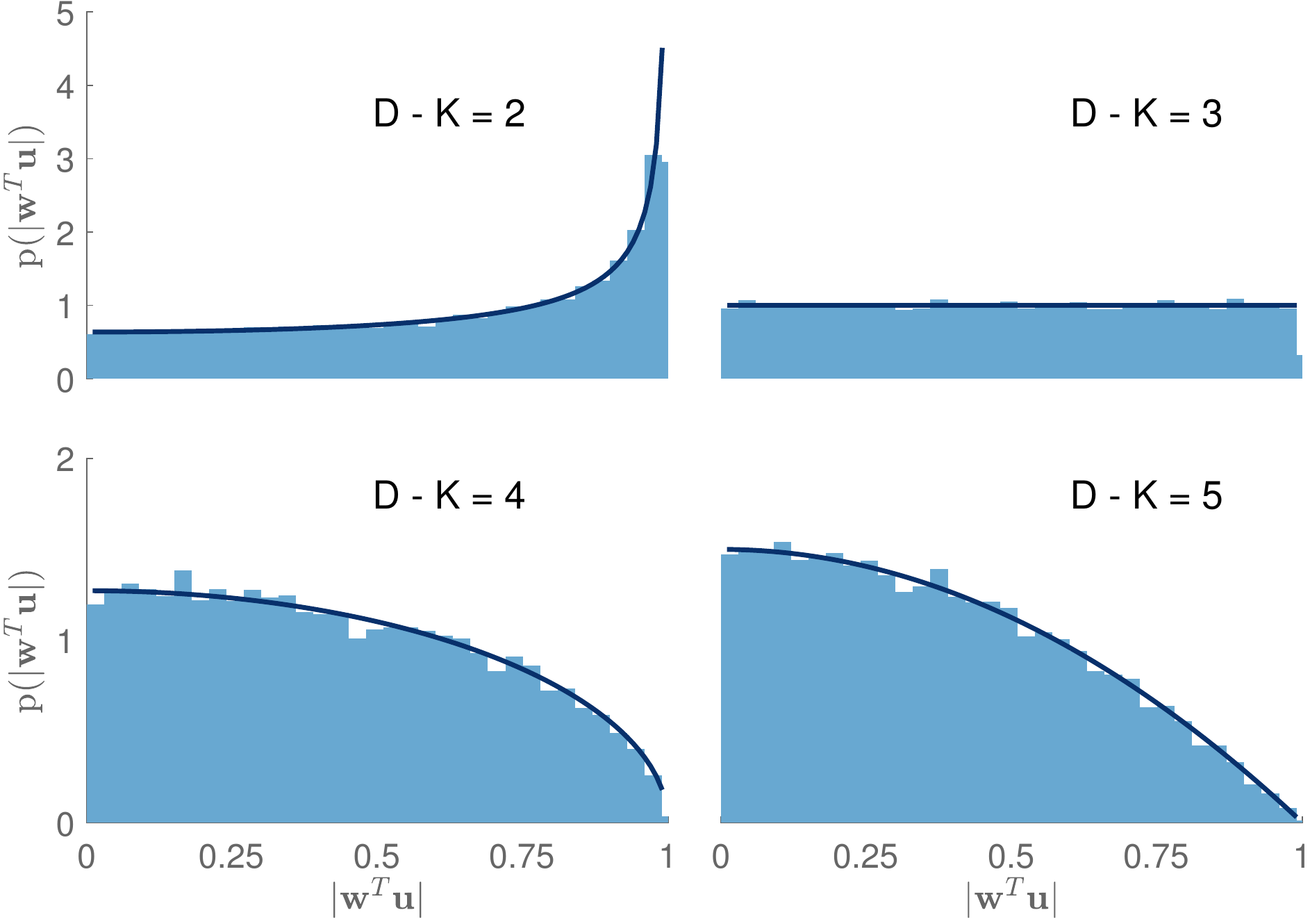}
	\caption{
		\label{fig:produit_scalaire}
		Empirical (light blue bars, computed from $20000$ samples) and theoretical (dark blue lines) pdf's associated with the cdf \eqref{eq:cdf_scalar_product} for $4$ different values of the dimension.
	}
\end{figure}

\begin{algorithm}
	\caption{Selecting the number of relevant directions}
	\label{alg:select_number_of_directions}
	\KwIn{
		level of KS test; a collection of samples $\left\{ \Vfac_1^{(t)}, \ldots \Vfac_\dimobs^{(t)}, \MATibp^{(t)}  \right\}_{t=n_{\mathrm{burn}}+1}^{T_{MC}}$ generated by Alg.~\ref{alg:gibbs}.
	}
	\BlankLine
	For each iteration, relabel the directions $\Vfac_\indcoef^{(t)}$ w.r.t. their frequency of activation, given by $\MATibp^{(t)}$\;
	Sample $\bfu_{1}\ldots\bfu_{\dimobs} \overset{\mathrm{i.i.d.}}{\sim} \stiefel{\dimobs}{1}$ \;
	\For{$K \gets 1$ \KwTo $\dimobs-1$} {
		\For{$t \gets n_{\mathrm{burn}}+1$ \KwTo $n_{\mathrm{burn}} + n_{\mathrm{iter}}$} {
			Let $\bfN_K$ be a basis of the orthogonal of $\Vfac_1^{(t)}\ldots \Vfac_{K}^{(t)}$ \;
			Compute $\omega_{K+1}^{K\,(t)} \triangleq  \Vert \bfN_K\transp \bfu_{K+1} \Vert^{-1} | \Vfac_{K+1}^{(t)\, T} \bfN_K\transp \bfu_{K+1} |,\ldots$ $\omega_\dimobs^{K\,(t)} \triangleq \Vert \bfN_K\transp \bfu_{\dimobs} \Vert^{-1} | \Vfac_\dimobs^{(t)\, T} \bfN_K\transp \bfu_\dimobs |$ \;
		}
		Stack the $\omega_{K+1}^{K\,(t)}, \ldots \omega_\dimobs^{K\,(t)}$ into a single collection of samples in view of Kolmogorov-Smirnov's test \;
		\If{$\calH_K$ is not rejected}{
			$\hat{\dimcoef}_{\mathrm{KS}} = K$ \;
			break\;
		}
	}
	\BlankLine
	\KwOut{$\hat{\dimcoef}_{\mathrm{KS}}$, an estimator of the number of relevant components.}
\end{algorithm}

We propose to use Theorem~\ref{th:distribution_scalar_product} to design the following Kolmogorov-Smirnov test of goodness-of-fit applied to the matrices $\left\{\MATfac^{(t)}\right\}_{t=n_{\mathrm{bi}}}^{n_{\mathrm{mc}}}$ generated by the Gibbs sampler detailed in Algo. \ref{alg:gibbs}. For a given candidate $\mathcal{A}$ of  $K$ indices associated with the subspace spanned by $\MATfac_{\mathcal{A}}$, one can test whether the remaining set $\bar{\mathcal{A}}$ of indices corresponds to directions $\MATfac_{\bar{\mathcal{A}}}$ uniformly distributed over the orthogonal group $\mathcal{O}_{D-K}$. Thanks to Theorem~\ref{th:distribution_scalar_product} this is equivalent to test whether the absolute scalar products \eqref{eq:omega_absolute}  are distributed according to \eqref{eq:cdf_scalar_product}. Note also that the random variables $\left\{\omega_\ell\right\}_{ \ell \in \bar{\mathcal{A}}}$ form a set of identically distributed components of a $L$-dimensional random vector $\boldsymbol{\omega}_{\bar{\mathcal{A}}}$. This permits to use a single statistical test to be performed for each dimension candidate $K$ iteratively in increasing or decreasing order, rather than $\dimobs - K$ multiple tests.
The null hypothesis is defined as
\begin{equation}
    \label{eq:KS_test}
	\calH_0^{(K)}: \omega_\ell \overset{\mathrm{cdf}}{\sim} \eqref{eq:cdf_scalar_product},\quad \forall \ell \in \bar{\mathcal{A}}=\{D-K+1,..., D\}
\end{equation}
Obviously, if this null hypothesis is accepted for a given set $\bar{\mathcal{A}}$ of $D-K$ indices, it will be accepted for any subset of lower dimension. Conversely, if this null hypothesis is rejected for some $K$ and a given set $\bar{\mathcal{A}}$ of $D-K$ indices, it will be definitely rejected for any superset of $\bar{\mathcal{A}}$, that is for subspace dimensions smaller than $K$. Since the objective of the proposed procedure is to identify an a priori small number $K$ of relevant components (and not a lower or upper bound), this hypothesis should be tested for an increasing number $K$ of active components. As a result, the following estimator $\widehat{K}_{\mathrm{KS}}$ of the number of active components is finally proposed:
\begin{equation} \label{eq:estimators:def_kolmo_estimators}
	\Kkolmo = \min \left\{ K \in \left\{ 0,\dotsc, \dimobs\right\}\; |\; \calH_0^{(K)} \text{ is accepted} \right\} .
\end{equation}
By convention, $\calH_0^{(D)}$ is accepted when $\calH_0^{(K)}$ has been rejected for all $K \in \left\{ 0,\dotsc, \dimobs-1\right\}$: thus the model would identify data to white Gaussian noise with no special direction. Algo.~\ref{alg:select_number_of_directions} describes the full procedure.
\subsection{Estimating other parameters} \label{subsec:model:estimators}

This section discusses the derivation of estimates associated with the remaining parameters,  other than the dimension $K$ of the subspace. Regarding the orthonormal matrix $\MATfac$ of which the $K$ first columns span the signal subspace, it is not recommended to use a simple average of the samples $\MATfac^{(t)}$ generated by the MCMC algorithm to approximate the minimum mean square error (MMSE) estimator. Indeed, the Markov chain targets a highly multimodal distribution with modes that depend on the current state of the dimension $K^{(t)}$. In particular, at a given iteration $t$, the last $D-K^{(t)}$ columns of $\MATfac^{(t)}$ are directly drawn from a uniform prior. One alternative is to compute the MMSE estimator conditioned upon an estimate $\widehat{K}$ of the relevant dimension. This can be easily done by averaging the samples $\MATfac^{(t)}$ corresponding to the iterations $t$ for which $K^{(t)}=\widehat{K}$. A similar procedure applies for the binary matrix $\MATibp$. Note that in the specific context of parametric subspace estimation, other Bayesian estimators have been proposed by \citet{Besson2011,Besson2012}.

\noindent{\bf Remark: }
the posterior distribution of the scale parameters $\boldsymbol{\zellner} = \left\{\zellnervar_1,\dotsc,\zellnervar_K\right\}$, where the matrix $\MATfac$ has been marginalized, cannot be derived analytically.  This posterior distribution can be derived explicitly in some very particular cases only, assuming that the binary matrix $\MATibp$ is the $\dimcoef\times\Nobs$ matrix $\boldsymbol{1}_{\dimcoef,\Nobs}$ with only $1$'s everywhere, see App.~\ref{app:deltapost} for details. The resulting posterior involves a generalized hypergeometric function of two matrices that could be used as a measure of mismatch between the magnitudes of the eigenvalues of covariance matrices. We leave this open question for future work.

\section{Performance assessment of BNP-PCA} \label{sec:easyExperiment}

\begin{figure*}[ht!]

\begin{tikzpicture}
	\node[right] at (0, 0){\includegraphics[width=\columnwidth]{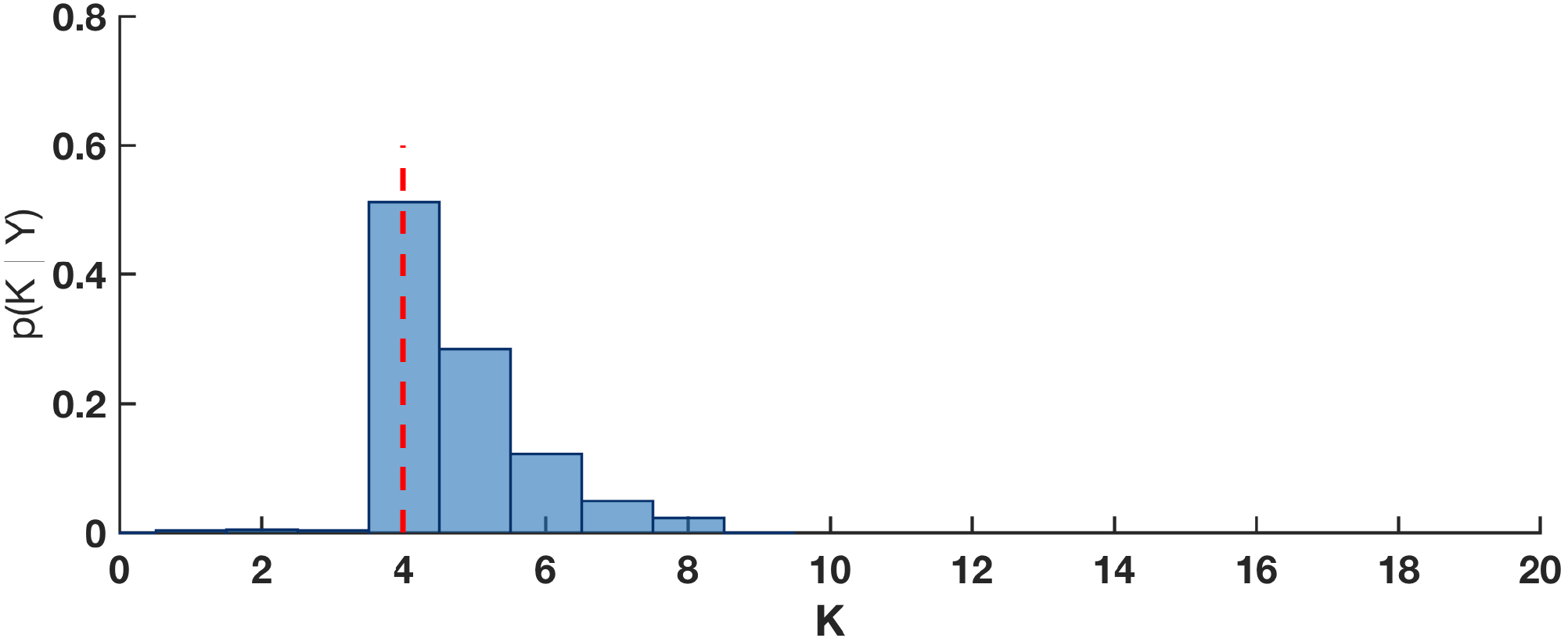}};
	\node[right] at (9, 0){\includegraphics[width=\columnwidth]{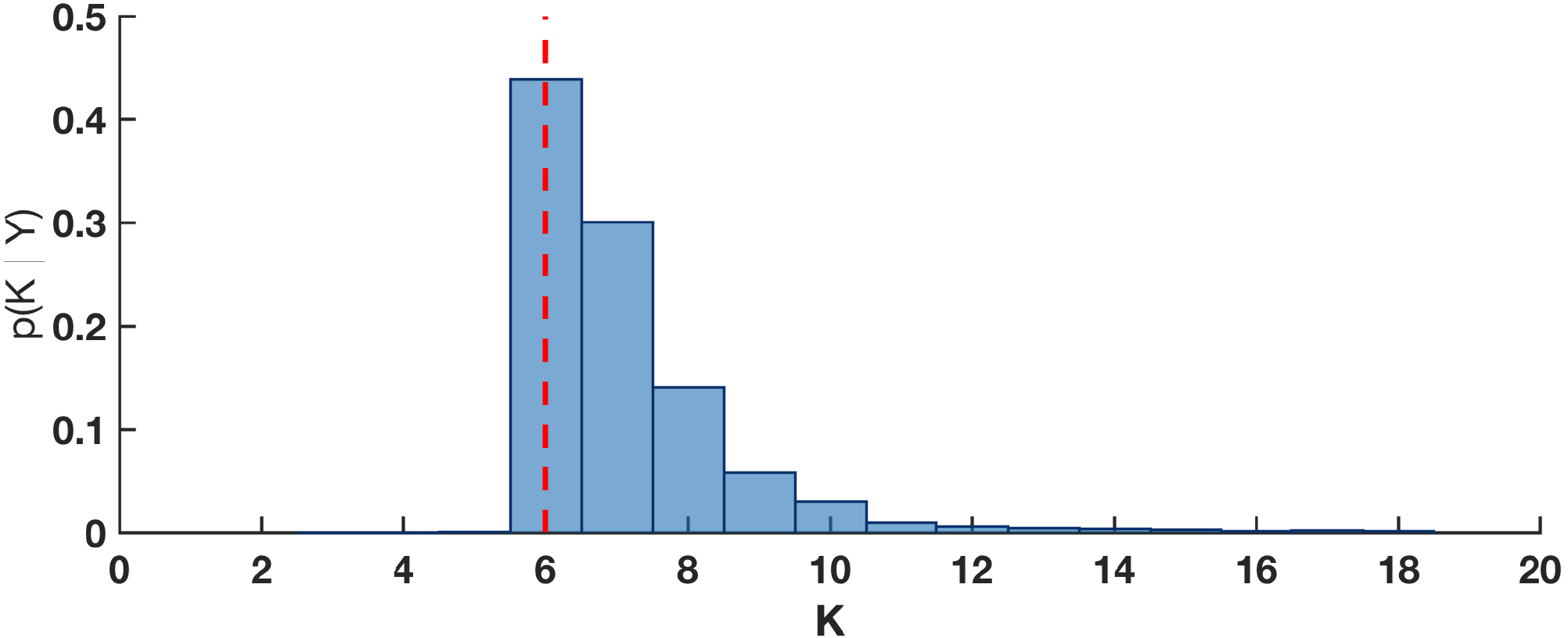}};
	\node[right] at (0, -4){\includegraphics[width=\columnwidth]{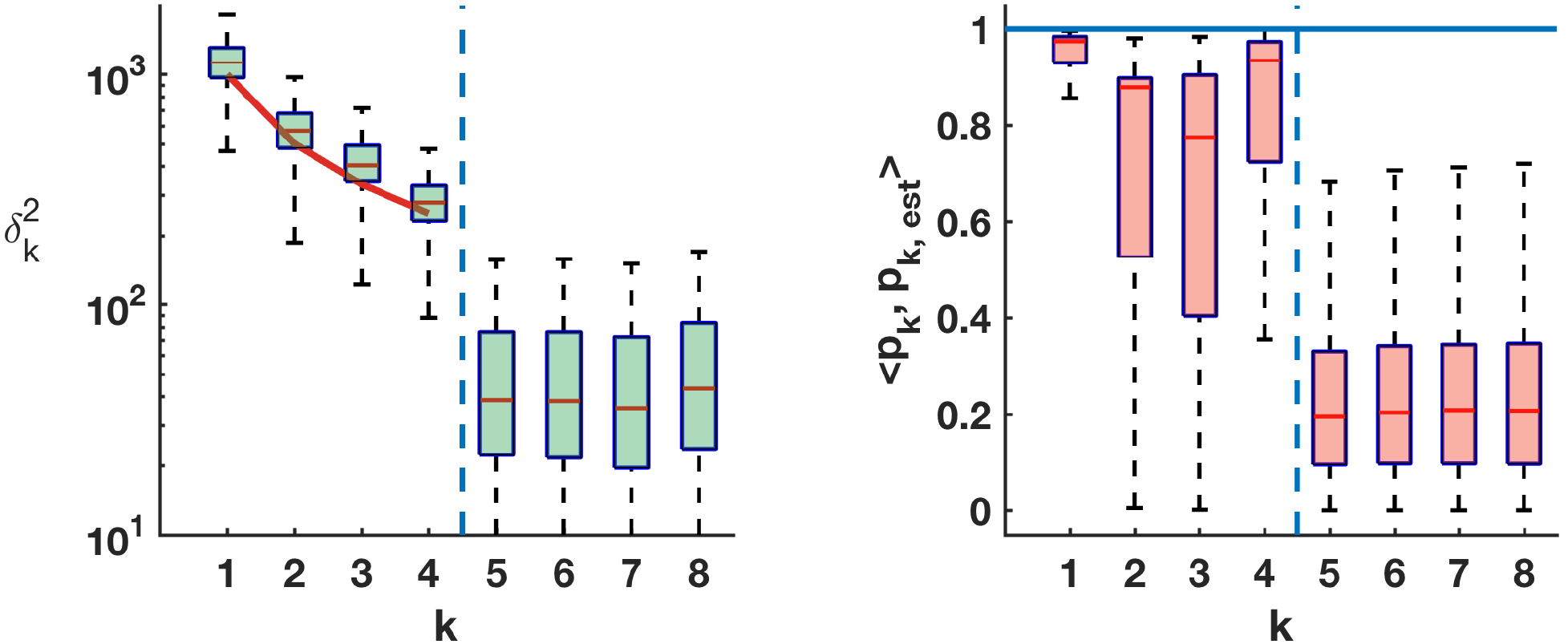}};
	\node[right] at (9, -4){\includegraphics[width=\columnwidth]{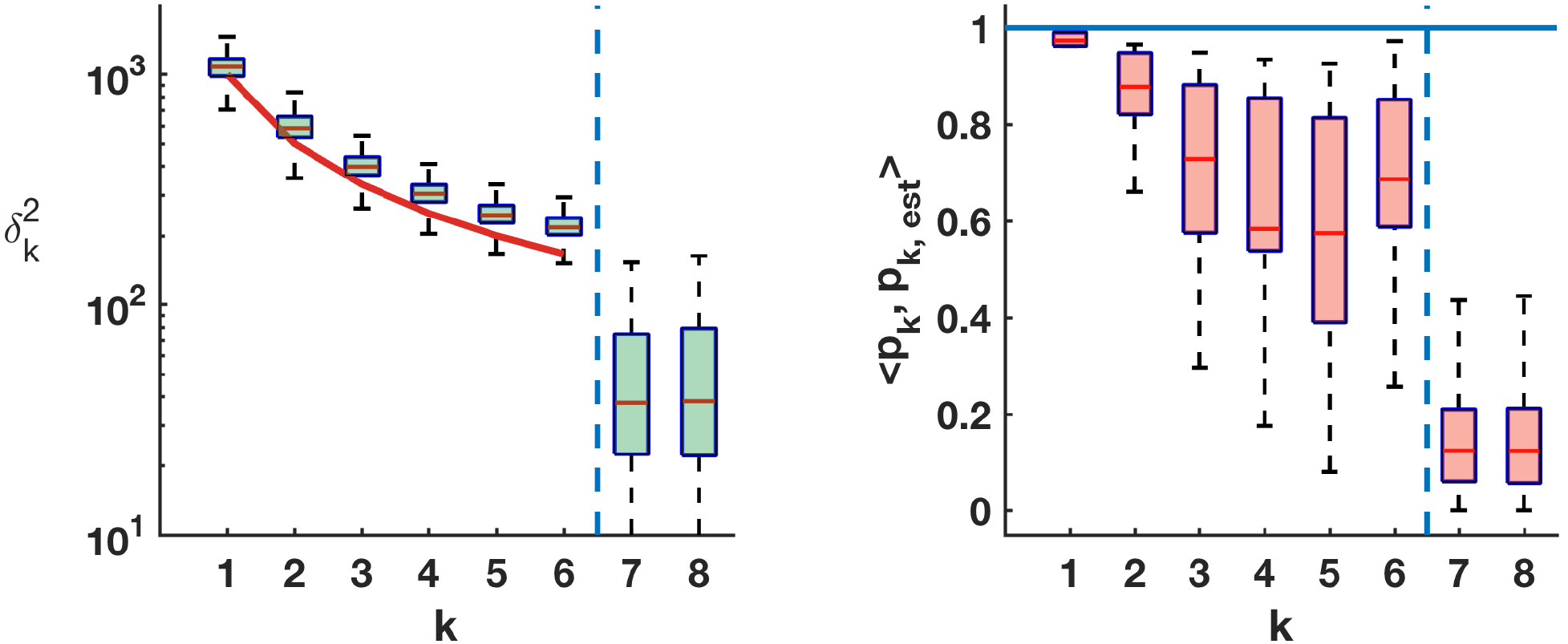}};
	\node at (4.5, 1.5) {\textsf{\textbf{(a)}}};
	\node at (13.5, 1.5) {\textsf{\textbf{(b)}}};
	\node at (2.5, -2.1) {\textsf{\textbf{(c)}}};
	\node at (7, -2.1) {\textsf{\textbf{(d)}}};
	\node at (11.5, -2.1) {\textsf{\textbf{(e)}}};
	\node at (16, -2.1) {\textsf{\textbf{(f)}}};
\end{tikzpicture}
	\caption{
		\label{fig:figure1}
		Top :
		posterior distribution of $\dimcoef$ for (a) $\dimobs=16, \Nobs=100$,  and (b) $\dimobs=36, \Nobs=500$.
		Bottom :
		posterior distributions of (c) \& (e) scale factors $\zellner_\dimcoef^2$  and (d) \& (f) dispersion of the projection ${\widehat{\MATfac}\transp} \MATfac$ for $\dimobs=16, \Nobs=100$ and $\dimobs=36, \Nobs=500$, respectively.
		The red lines indicates the true values of $\zellnervar_1\dots \zellnervar_\dimcoef$.
	}
\end{figure*}

The performance of the proposed BNP-PCA is assessed on datasets simulated according to the linear model
\begin{equation} \label{eq:easyExp:generativeModel}
	\Vobs_\nobs = \bfH \bfu_\nobs + \bfe_\nobs
\end{equation}
where $\bfe_\nobs$ is an additive Gaussian noise of covariance matrix $\noisevar \indmat_\dimobs$ and the quantities $\bfH$ and $\bfu_\nobs$ are specified as follows. First, for a given dimension $\dimobs$ of the observations, $\dimcoef$ orthonormal directions are gathered in a $\dimobs\times \dimcoef$ matrix $\bfH$ which is uniformly generated on the Stiefel manifold $\stiefel{\dimobs}{\dimcoef}$. Then, $\Nobs$ representation vectors $\bfu_1,\dotsc, \bfu_\Nobs$ of dimension $\dimcoef$ are identically and independently generated according to a centered Gaussian distribution with a diagonal covariance matrix $\boldsymbol{\Sigma} = \mathrm{diag}\left\{\delta_1^2\sigma^2,\dotsc,\delta_{\dimcoef}^2\sigma^2\right\}$ where the scale factors $\delta_1^2,\dotsc,\delta_K^2$ control the relevance of a particular direction.
Equivalently, by choosing different values for the scale factors, this model also conveniently permits to consider the case of an anisotropic noise corrupting an isotropic latent subspace. In the following, the choice of these scale factors will be specified in four typical scenarios.

Since each scale factors $\delta_k^2$ controls the signal-to-noise ratio in each direction, a unique value $\noisevar=0.01$ of the noise variance is considered without loss of generality. Several dimensions $\dimobs$ and $\dimcoef$ are considered for various numbers of observations $\Nobs$. 
The proposed Gibbs sampler has been run during $1000$ iterations after a burn-in period of $100$ iterations.

\subsection{Scale factors and alignment of components}

The performances of the proposed algorithm have been first evaluated on various simulated datasets.
As an illustration, we report here the results on $2$ datasets corresponding to $(\dimobs=16, \dimcoef=4, \Nobs=100)$ and  $(\dimobs=36, \dimcoef=6, \Nobs=500)$ and where the scale coefficients $\delta_k^2$ are defined as proportional to $1/k$.

\noindent Fig.~\ref{fig:figure1}(a) \& (b) show the posterior distributions of $\dimcoef$ for $(\dimobs=16,\Nobs=100)$ and $(\dimobs=36, \Nobs=500)$, respectively.
We observe that the maximum of the two posterior histograms correspond to the expected dimension, \textit{i.e.}, $K=4$ for $\dimobs=16$ and $\dimcoef=6$ for $\dimobs=36$.
Note that this estimator corresponds to the marginal maximum a posteriori estimator defined by Eq.~\eqref{eq:K_mMAP}.
These two examples suggest that the marginal MAP estimator $\dimcoef | \MATobs$ seems to be consistent since it is able to recover the expected dimension. This is in contrast with the behaviour of the conditional MAP estimator $\dimcoef | \MATobs, \alpha$ that is known to be inconsistent from Theorems~\ref{th:inconsistance} and~\ref{th:inconsistance_severe}. Section~\ref{subsec:easyExperiment:K_func_N} will come back to this question in more details.
We do not comment on the behaviour of $\Kkolmo$ based on KS tests here: in such simple scenarios, $\Kkolmo$ and $\Kkolmo$ always give the same results.


\noindent Fig.~\ref{fig:figure1} (c) \& (e) show the posterior distributions of the 8 first scale factors.
Fig.~\ref{fig:figure1}(d) \& (f) show the alignment of the true $\Vfac_k$ with the estimated $\widehat{\Vfac}_k$; see Fig.~\ref{fig:figure1}(c)\&(d) for $\dimobs=16,\Nobs=100$ and Fig.~\ref{fig:figure1}(e)\&(f) for $\dimobs=36, \Nobs=500$.
The alignment is measured by the scalar product $\langle \Vfac_k,\widehat{\Vfac}_k\rangle$ between each column of $\MATfac$ and its estimate.
No ordering problem is expected here since the variances are sufficiently different in every direction.
In both cases, it appears that scale factors are correctly identified.
We observe that inferred directions correspond to actual principal components with an alignment typically higher than $0.8$ on average.
All other components, for $k\geq 5$ on Fig.~\ref{fig:figure1}(d) and $k\geq 7$ on Fig.~\ref{fig:figure1}(f)), are considered as inactive since associated to components with comparable factors and much lower alignment.
This observation motivated the procedure proposed in Section~\ref{subsec:selecting_true_K} elaborated on KS tests to build the estimator $\Kkolmo$, see Eq.~\eqref{eq:estimators:def_kolmo_estimators}. Recall that $\Kkolmo$ will be especially useful when the signal to noise ratio is close to 1 for some components, that is $\delta_k^2\simeq 1$.

These first experiments show that the proposed BNP-PCA is able to identify the relevant latent subspace through its dimension $K$ as well as principal components $\Vfac_k$ and their corresponding scale factors $\delta_k^2$. They also indicate that $\widehat{K}_{\mathrm{mMAP}}$ seems to be consistent in contrast with $\widehat{K}_{\mathrm{mMAP},\alpha}$, see Theorems~\ref{th:inconsistance} \& \ref{th:inconsistance_severe} of Section~\ref{subsec:estimators:behaviour_pk}.

\subsection{Marginal MAP estimator of the latent dimension} \label{subsec:easyExperiment:K_func_N}

This section experimentally investigates the behaviour of the marginal MAP estimator $\widehat{K}_{\mathrm{mMAP}}$ of the dimension of the latent subspace defined by \eqref{eq:K_mMAP}.
Note that this estimator is different from the marginal MAP estimator $\widehat{K}_{\mathrm{mMAP},\alpha}$ defined in \eqref{eq:K_mMAP_alpha} which was still conditioned upon $\alpha$. Indeed, in the Bayesian model proposed in Section \ref{sec:model}, a prior distribution is assigned to the hyperparameter $\alpha$ which is thus jointly inferred with the parameters of interest. While Theorem~\ref{th:inconsistance} of section~\ref{subsec:estimators:behaviour_pk} says that $\widehat{K}_{\mathrm{mMAP},\alpha}$ with fixed $\alpha$ is inconsistent, we will empirically show that $\widehat{K}_{\mathrm{mMAP}}$ seems to be consistent.

\begin{figure*}
	\centering
	\includegraphics[width=0.95\textwidth]{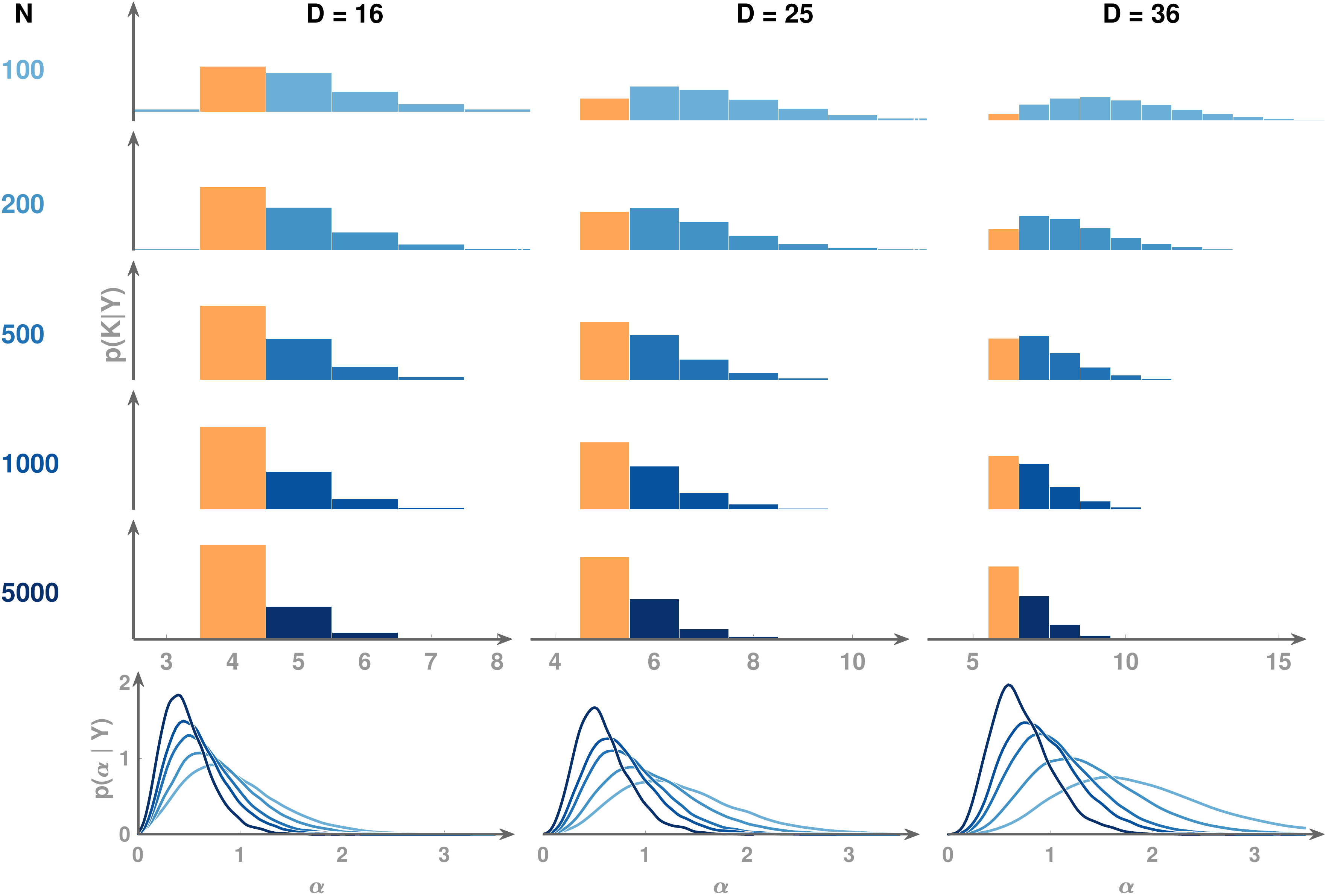}
	\caption{Empirical posterior probabilities $\mathrm{P}\left[\dimcoef=k|\MATobs\right]$ of the latent dimension for (left) $\dimobs=16$, (center) $\dimobs=25$, (right)$\dimobs=36$ and $\Nobs\in\left\{100,200,500,1000,5000\right\}$. The orange bars indicate the true dimension $K$ of the latent subspace. Bottom plots are the empirical marginal posterior distributions $f\left(\alpha | \MATobs\right)$ where the number of observations $\Nobs$ increases when the line color goes from light to dark blue lines.
	}
 	\label{fig:exp:synthetic:func_N}
\end{figure*}

Fig.~\ref{fig:exp:synthetic:func_N} shows the empirical posterior probabilities $\mathrm{P}\left[\dimcoef=k|\MATobs\right]$ when all the scaling factors have been fixed to values significantly higher than 1, such that $\zellnervar_\indcoef = 50 / \indcoef$, $1\leq \indcoef\leq K$.
Actual subspace dimensions are $K = \sqrt{D}$ for $\dimobs\in \left\{16, 25, 36\right\}$. This figure shows that, for $\dimobs=16$, the marginal MAP estimator $\widehat{K}_{\mathrm{mMAP}}$ correctly recovers the latent dimension for all values of $\Nobs$. The proposed model needs around $\Nobs=500$ observations for $\dimobs=25$, and $\Nobs=1000$ for $\dimobs=36$. All posteriors seem to concentrate around the true value $K = \sqrt{D}$ as the number of observations increases: these  numerical results suggest a consistent behaviour of the estimator.

These findings do not contradict Theorem~\ref{th:inconsistance} which states that the marginal MAP estimator of $\dimcoef$ is inconsistent {\em for fixed $\alpha$}. In contrast, sampling $\alpha$ jointly with the other parameters leads to a marginal MAP estimator $\widehat{K}_{\mathrm{mMAP}}$ which seems to be consistent, at least based on our numerical experiments. By examining the empirical marginal posterior distributions $f\left(\alpha | \MATobs\right)$ reported in Fig.~\ref{fig:exp:synthetic:func_N} (last row), one can note that this distribution seems to get closer to $0$ as the number of observations $\Nobs$ increases. Exploiting the fact that $\mathbb{E}[K]$ a priori scales as $\alpha \log(\Nobs)$, the posterior behaviour of the latent subspace dimension seems to result from a decreasing estimated value of $\alpha$, this is expected. Moreover, recall that Theorem~\ref{th:inconsistance} states that the marginal posterior probabilities $\mathrm{P}\left[ \dimcoef_N = k |\MATobs_\Nobs, \alpha \right]$ does not admit $1$ as a limit for any value $k$. However, it does not state that the mode cannot converge to the true value.

Finally, let us recommend that a certain care be taken anyway when resorting to these posterior probabilities. We have shown that the proposed estimator $\widehat{K}_{\mathrm{mMAP}}$ can exhibit a good asymptotic behaviour, but how this asymptote behaves still seems to depend both on the generative model and the experiment settings and is out of the scope of the present paper.

\subsection{The BNP-PCA of white Gaussian noise} \label{subsec:whitenoise}

In this experiment, the scaling parameters are all chosen as $\delta_k^2 = 0$, leading to observed measurements $\Vobs_\nobs$ ($\nobs=1,\dotsc,\Nobs$) only composed of white Gaussian noise. In this particular case, data do not live in a particular subspace. The purpose of this first basic experiment is to check whether the algorithm is able to detect that no component is relevant, i.e., $K=0$ since data behaves like white Gaussian noise.
More precisely, since the signal is only composed of isotropic noise, the empirical covariance matrix of the observed vectors verifies
\begin{equation}
	\lim_{\Nobs \rightarrow + \infty} \Nobs^{-1} \MATobs\MATobs\transp = \noisevar \indmat_\dimobs.
\end{equation}
According to Section \ref{subsec:selecting_true_K} and Theorem~\ref{th:distribution_scalar_product}, the posterior distribution of a potential active direction in \eqref{eq:posterior:Pk} should asymptotically tend to be
$ \propto\exp \left( \Vfac\transp \Vfac/4 \right)$ that is constant since $ \Vfac\transp \Vfac=1$ by definition: one expects that the $\Vfac_\indcoef$ be uniformly distributed on the unit sphere.
BNP-PCA estimates scale factors that are all comparable given the prior. Therefore BNP-PCA does not identify any special latent subspace in this case.


\begin{table}
\caption{
	\label{tab:exp:OnlyNoise:tests}
	Results of Kolmogorov-Smirnov goodness-of-fit tests at level 0.05 averaged over 20 Monte Carlo simulations when the signal is made of $\Nobs=500$ $D$-dimensional realizations of an isotropic Gaussian noise.	Scores reported in each column correspond to the probability of rejecting the null hypothesis for a subspace of candidate dimension $K$.
}
\centering
\vspace{3pt}
\begin{tabular}{l | cccccc}
\toprule[1.5pt]
	\textbf{$K$} &
	\textbf{0} &
	\textbf{1} &
	\textbf{2} &
	\textbf{3} &
	\textbf{4} &
	\textbf{5} \\
	\midrule
    $\dimobs = 9 $ & 0.05 &    0.05 &   0.05 &    0.05 &    0.05 &    0 \\
    $\dimobs = 16$ & 0.05 &    0    &   0.05 &    0.05 &    0    &    0 \\
    $\dimobs = 25$ & 0.05 &    0.1  &   0.05 &    0.1  &    0.05 &    0 \\
    $\dimobs = 36$ & 0.05 &    0.05 &   0.05 &    0.05 &    0.05 &    0.05 \\
\bottomrule[1.5pt]
\end{tabular}
\end{table}

Table~\ref{tab:exp:OnlyNoise:tests} shows the results provided by the Kolmogorov-Smirnov (KS) goodness-of-fit test described in Section \ref{subsec:selecting_true_K}. More precisely, for $\Nobs=500$ and $\dimobs \in \left\{9, 16, 25, 36\right\}$, Table~\ref{tab:exp:OnlyNoise:tests} reports the probability of rejecting the null hypothesis $\mathcal{H}_0^{(K)}$ in \eqref{eq:KS_test} for candidate dimensions $K\in\left\{0,\dotsc,5\right\}$ of the latent subspace, i.e., $L=D-K\in\ldots\left\{\dimobs,\dotsc,\dimobs-5\right\}$. These results computed from $20$ Monte Carlo simulations show that the null hypothesis is very often rejected with a probability of the order of $0.05$, which corresponds to the chosen rejection level of the KS test here: it is considered as accepted (not rejected). Similar results are obtained for $K\in\left\{6,\dotsc,\dimobs\right\}$. As expected, the estimator $\Kkolmo$ defined by \eqref{eq:estimators:def_kolmo_estimators} well recovers the actual dimension of the latent subspace, i.e., $K=0$ here since the data is simply white Gaussian noise only.

\subsection{Influence of the distribution of  scaling factors} \label{subsec:anisotropic}

The third experiment aims at investigating two aspects of BNP-PCA. The first question is how far principal components are well recovered. The second aspects concerns the limitations of the proposed method when some scaling factors $\delta_k^2$ are below $1$, leading to poorly relevant directions of the latent subspace with respect to the noise level. More precisely, $\Nobs$ measurement vectors have been generated according to the model \eqref{eq:easyExp:generativeModel} with $\Nobs\in\left\{200,2000\right\}$, $\dimobs=16$ and $K=16$  with scaling factors $\delta_k^2 = 10 / k^{2.2}$ ($k=1,\dotsc,K$), such that the first 5 scaling factors are $[10, 2.2, 0.9, 0.5, 0.3]$; only 2 are larger than 1.
This setting permits to play with individual signal-to-noise ratios specified in each direction. Since the scaling factors $\delta_k^2$ are lower than $1$ for $k \geq 3$, not all directions are expected to be recovered. 
%

\begin{figure}
	\centering
	\includegraphics[width=\columnwidth]{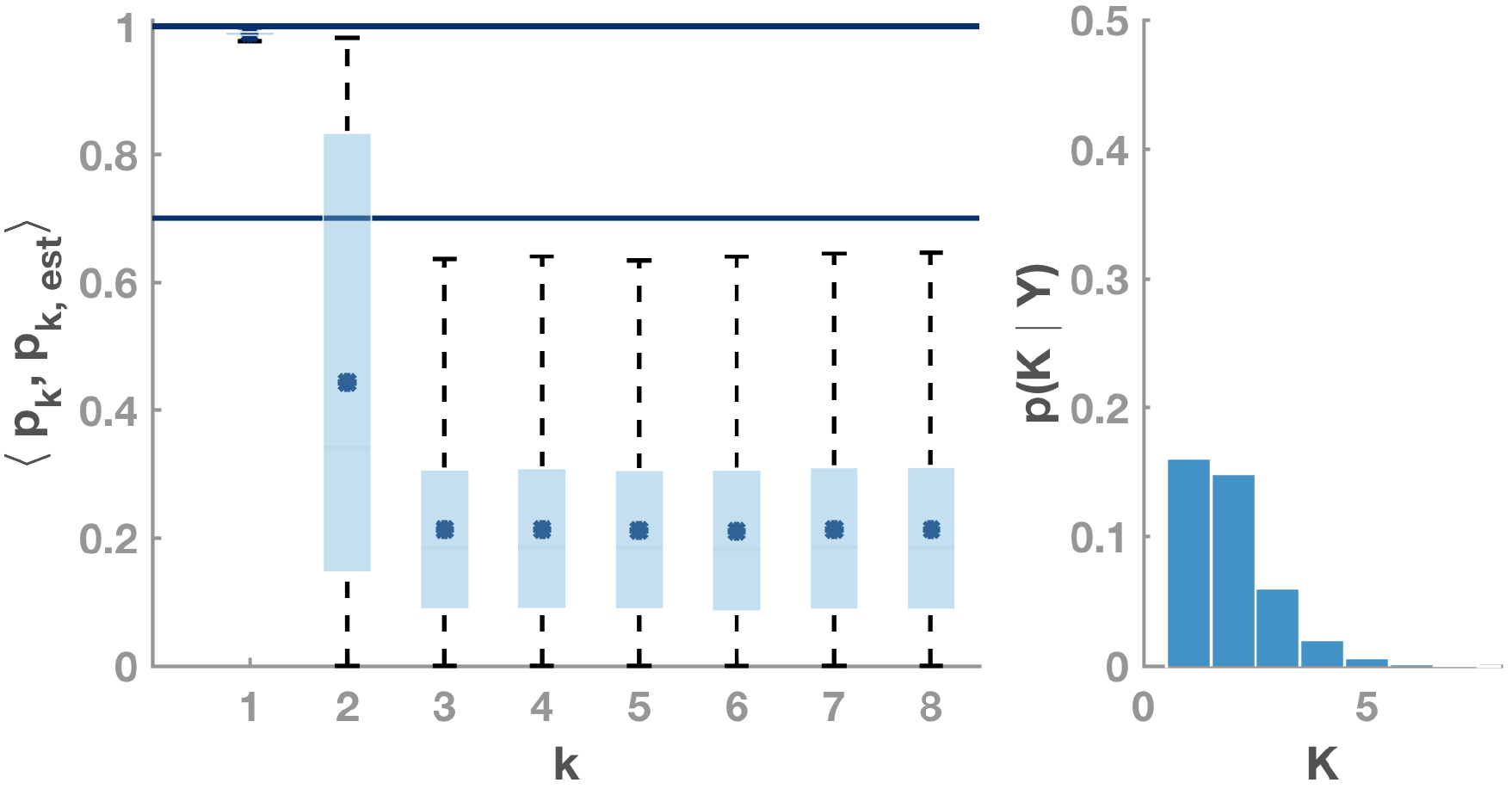}
	\includegraphics[width=\columnwidth]{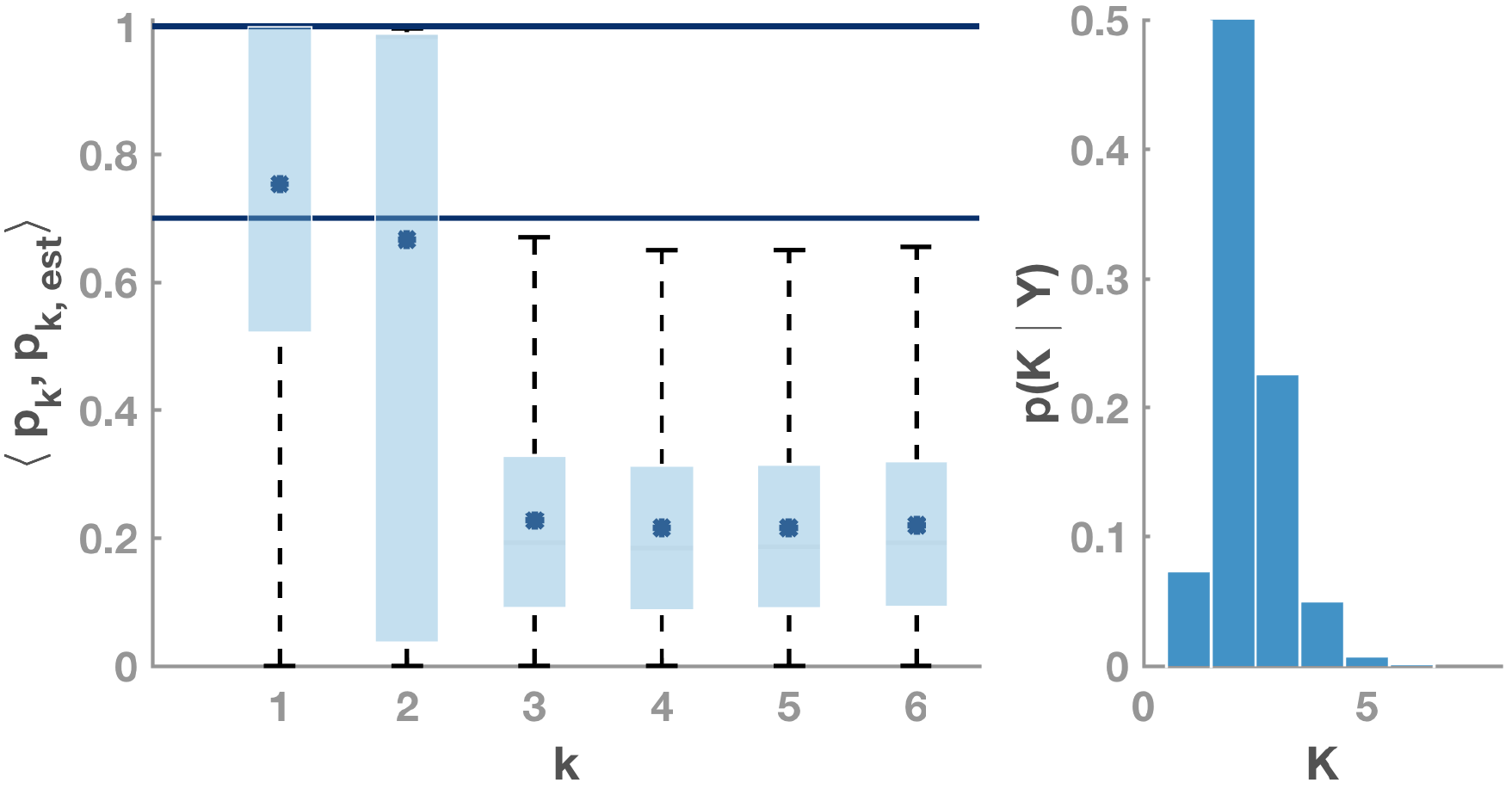}
	\caption{	\label{fig:anisotropic:results:posteriors}
		Marginal posterior distributions in case of signal with anisotropic noise, for $\dimobs=16$, $\Nobs=200$(top) and $\Nobs=2000$(bottom).
	}
\end{figure}

Fig. \ref{fig:anisotropic:results:posteriors} (right) shows the empirical marginal posterior probability of the latent dimension. These probabilities lead to marginal MAP estimators \eqref{eq:K_mMAP} of the latent dimension equal to $ \widehat{K}_{\mathrm{mMAP}}=2$ for both cases ($N=200$ and $N=2000$). The alternative estimator $\Kkolmo$ of the latent subspace derived from the Kolmogorov-Smirnov test (see Section \ref{subsec:selecting_true_K}) leads to estimates between $2$ (65\% provides $\Kkolmo=2$ for $\Nobs=200$) and $3$ (95\% provides $\Kkolmo=3$ for $\Nobs=2000$). These experiments indicate that BNP-PCA fails to detect principal components weaker than the noise level.

Fig.~\ref{fig:anisotropic:results:posteriors} (left) depicts the estimated inner products $\langle \Vfac_{k}, \hat{\Vfac}_{k} \rangle$ and corresponding confidence intervals computed from $50$ Monte Carlo simulations where $\hat{\Vfac}_{k}$ denote the estimated direction vectors. A high score (like a cosine) indicates a good alignment of the vectors, thus a correct recovery of the corresponding latent direction. This figure shows that, for $N=200$ (top), the proposed model accurately identifies the first component only among the two expected from $ \widehat{K}_{\mathrm{mMAP}}=2$. For larger $N=2000$ (bottom) the alignment is better and the 2 predicted components are well recovered as attested by the good alignment between the $\hat{\Vfac}_{k}$ and $\Vfac_k$. However, in both cases, the proposed strategy is not able to extract components with scaling factors $\delta_k^2$ smaller than $1$: they are identified \textcolor{red}{to} noise, as expected from signal-to-noise ratios.



\section{Applications} \label{sec:experimentsReal}

\subsection{BNP-PCA and clustering} \label{subsec:clustering}

To illustrate the flexibility of the proposed model, a simple experiment where the dimension reduction is combined with a linear binary classifier is presented.
The representation coefficients in Eq.~\eqref{eq:model:model} are now modeled by a mixture of two Gaussian distributions corresponding to 2 distinct clusters
\begin{equation}
\label{eq:GMM}
	\forall \nobs, \quad \Vcoef_{\nobs} {}\sim{} \pi \; \mathcal{N}\left( \boldsymbol{\mu}_0, \boldsymbol{\Delta}_0 \right) + (1 - \pi) \; \mathcal{N}\left( \boldsymbol{\mu}_1, \boldsymbol{\Delta}_1 \right),
\end{equation}
where $\boldsymbol{\mu}_i = \left[\mu_{i,1},\dotsc,\mu_{i,K}\right]\transp$ and $\boldsymbol{\Delta}_i = \mathrm{diag}\left\{\zellnervar_{i,1},\dotsc, \zellnervar_{i,K} \right\}$ for $i\in\left\{0, 1 \right\}$ are respectively the mean and the covariance matrix associated with each class.
A common centered Gaussian distribution is used as the prior distribution for the mean vectors $\boldsymbol{\mu}_i$ ($i \in \left\{0,1\right\}$) assumed to be a priori independent, i.e., $\boldsymbol{\mu}_i\sim\mathcal{N}\left( \boldsymbol{0}, s^2 \indmat \right)$. Note that the use of non-informative priors are prohibited here due to posterior consistency. Additionally, a binary label vector $\boldsymbol{\eta} = \left[\eta_1,\dotsc,\eta_\Nobs\right]^T$ which indicates whether the $\nobs$th observation belongs to the class ${\cal C}_0$ or ${\cal C}_1$ is assigned equiprobable prior probabilities and will be jointly estimated with the parameters of interest. Analytical marginalization w.r.t. to the scale factors remains tractable. All prior distributions are conjugate, yielding conditional posterior distributions that can be easily derived and sampled as described in Section~\ref{subsec:model:inference}.

\newcommand{\KMa}{kmean-np}
\newcommand{\KMb}{kmean-pca}
\newcommand{\EMa}{EM-nz}
\newcommand{\EMb}{EM-pca}
\newcommand{\USa}{us-1}
\newcommand{\USb}{us-2}


~ \\
\noindent \textbf{Results on a subset of the MNIST database.}
The performance of the proposed algorithm is illustrated on a subset of the MNIST database\footnote{Available online at \url{http://ufldl.stanford.edu/wiki/index.php/Using_the_MNIST_Dataset}}, obtained by extracting the first $200$ images associated with the digits $6$ and $7$.
Each image is encoded as a vector in lexicographic order where pixels with null variance (i.e., pixels mainly located in the image corners) have been removed, leading to observation vectors of dimension $D=572$. The objective of this experiment is to evaluate the need and impact for dimension reduction for this binary classification task. The results provided by the proposed method are compared with those obtained by using an expectation-maximization (EM) algorithm\footnote{Available through the \textit{gmdistribution} class of MATLAB.} as well as an MCMC algorithm, both inferring the parameters associated with the conventional Gaussian mixture model \eqref{eq:GMM} described above. Both algorithms, denoted respectively by GMM-EM and GMM-MCMC, are preceded by a supervised dimension reduction preprocessing which consists in computing the first $K$ principal components, for a wide set of dimensions $K$. We emphasize that the proposed BNP-PCA approach combined to an MCMC algorithm for inference addresses jointly the dimension reduction and classification tasks as well as it identifies the dimension of the relevant latent subspace and estimates the noise level.

To overcome the problem of label switching inherent to MCMC sampling of mixture models, the samples generated from the proposed Bayesian nonparametric approach and the Bayesian parametric GMM-MCMC algorithms are postprocessed appropriately \citep[Chapter 6-4]{Marin2007}. More precisely, first, the two farthest observation vectors (in term of Euclidean distance) are assumed to belong to distinct classes. Gibbs sampler iterations leading to equal labels for these two observations are discarded. For remaining iterations, all the generated labels are reassigned in agreement with consistent labels for these two particular observations.

Classification performance is evaluated by the resulting labeling errors. All results have been averaged over $20$ Monte Carlo simulations.

\begin{figure}
	\centering
	\includegraphics[width=\columnwidth]{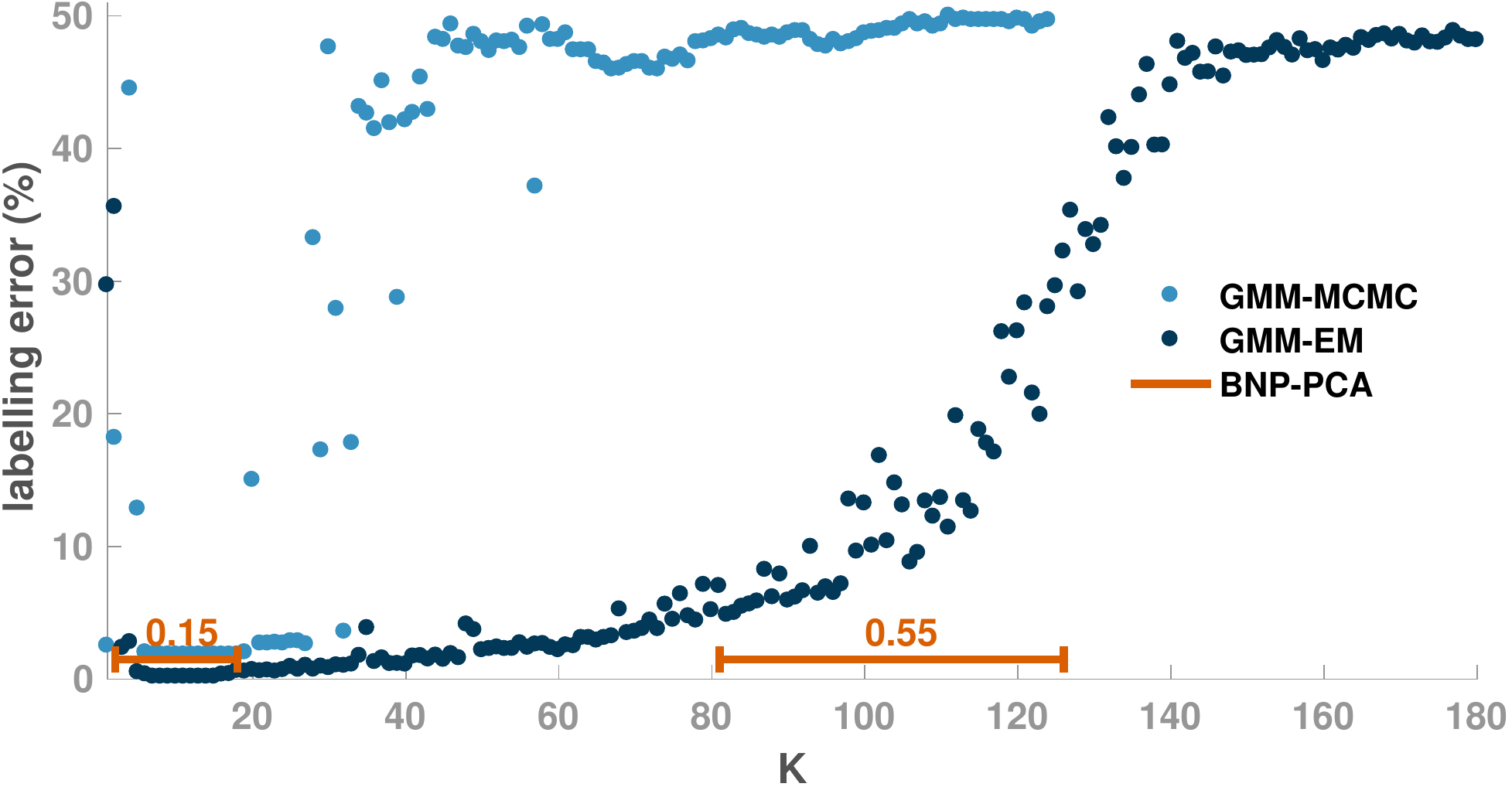}
	\caption{\label{fig:exp:clustering:mnist}
		Clustering results for the 200 first images of the MNIST database for digits $6$ and $7$.
	}
\end{figure}

Fig.~\ref{fig:exp:clustering:mnist} shows the clustering results for the 2 parametric methods compared to BNP-PCA.
Both parametric methods, GMM-EM and GMM-MCMC, show labeling errors close to $1$\% when using few principal components as input features, but exhibit a phase transition leading to error up to $50$\% when retaining too much principal components. Note that the phase transition occurs later for the EM-based algorithm that seems to be more robust, but a more elaborated MCMC method may have exhibited a similar performance.
The proposed Bayesian nonparametric method shows an average labeling error of about $1.5$\%. Fig.~\ref{fig:exp:clustering:mnist} indicates the typical ranges of values visited by the sampled latent dimension (brown lines). The intervals $\dimcoef \in [3, 18]$ and $\dimcoef \in [83, 130]$ correspond to $70$\% of the samples. It is noticeable that the two parametric methods reach their best performance when considering a number $\dimcoef$ of principal components belonging to the first interval.




\subsection{Hyperspectral subspace identification} \label{sec:experimentsHyperspectral}

As a second pratical illustration, the BN-PCA is employed to solve a key preprocessing task for the analysis of hyperspectral images. An hyperspectral image consists of a collection of several hundreds or thousands of $2$D images acquired in narrow and contiguous spectral bands. Such images can be interpreted as a collection of spectra measured at each pixel location. A classical objective is the recovering of spectral signatures of the materials that are present in the scene as well as their spatial distributions over the scene. A common assumption in spectral unmixing is to consider that each measured spectrim is a noisy convex combination of the unknown elementary spectral signatures called {\em endmembers}. The combination coefficients correspond to the unknown proportions to be estimated. Thus this so-called spectral unmixing can be formulated as a classical blind source separation or nonnegative matrix factorization problem. One crucial issue lies in the fact that the number $R$ of endmembers (i.e., the order of decomposition/factorization) present in the image is generally unknown in most applicative scenarios. However, under the hypothesis of a linear mixing model, measurements should lie in a $K$-dimensional linear subspace with $K=R-1$. As a consequence, most of the spectral unmixing techniques first estimate the relevant latent subspace by a dimension reduction step such as PCA. Then one usually considers \citep{Bioucas2012jstars} that the number of materials present in the scene is $R=K+1$. Precisely, the proposed BNP-PCA can identify the number $R$ of components that are significant in an hyperspectral image.

A real hyperspectral image, referred to as ``Cuprite hill'' and acquired by the Airborne Visible/Infrared Imaging Spectrometer (AVIRIS) over Cuprite, Nevada, is considered. The image of interest consists of $1250$ pixels observed in $190$ spectral bands after spatial subsampling in horizontal and vertical directions of a factor 2 and after removing the spectral bands of low SNR  typically corresponding to the water absorption bands. Then the hyperspectral image has been whitened according to the noise covariance matrix estimated by the strategy described by \cite{bioucas-dias2008}.

The proposed BNP-PCA based method is compared to the generic methods referred to as L-S and OVPCA introduced by \citet{minka2000_rd} and \citet{Smidl2007}, respectively, as well as to the hyperspectral-specific subspace identification algorithm HySime \citep{Bioucas2008}. The proposed Gibbs sampler has been run during $1100$ iterations including a burn-in period of $100$ iterations.


The HySime algorithm estimates a hyperspectral subspace of dimension $\widehat{K} = 10$ while L-S and OVPCA lead to $\widehat{K}=25$ and $\widehat{K} = 23$, respectively. There is no oracle correct number of materials or dimension of the latent subspace. Examining the crude mapping of the materials conducted by \cite{Clark1993} and \cite{Clark2003} permits to state that it is highly unlikely that more than $15$ materials are present in the considered region of interest.  Specialists generally agree about a number of components between 10 and 15. It appears that both HySime and OVPCA overestimate the number of endmembers
Using BNP-PCA on the same dataset, the marginal MAP estimator defined by \ref{eq:K_mMAP} yields $\widehat{K}_{\mathrm{mMAP}} = 25$ while the implementation of the Kolmogorov-Smirnov goodness-of-fit test detailed in Section \ref{subsec:selecting_true_K} leads to a latent subspace dimension estimate $\widehat{K}_{\mathrm{KS}}=13$ which is quite coherent with the expected value.


\begin{figure}
	\centering
	\includegraphics[width=\columnwidth]{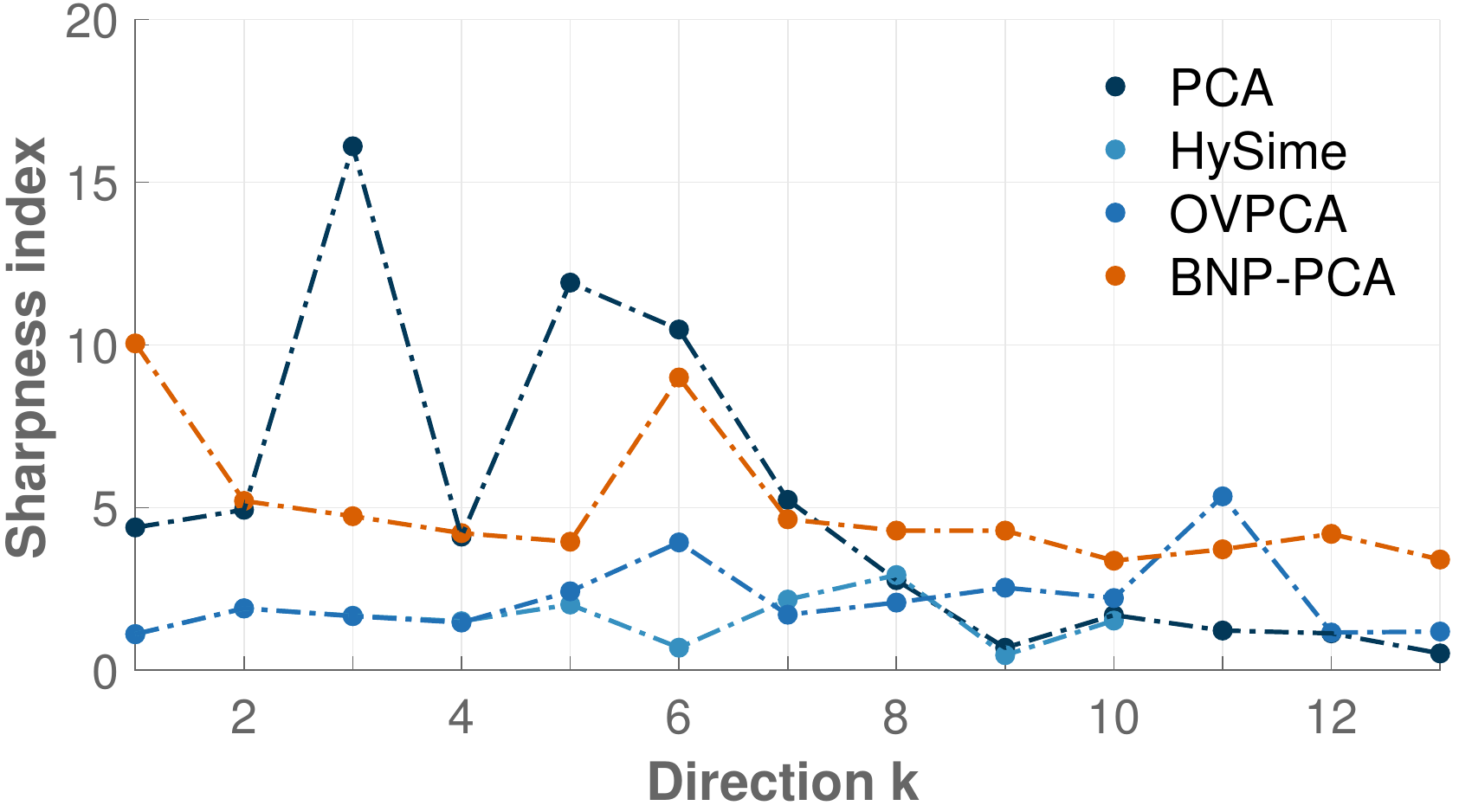}
	\caption{Sharpness index of the images resulting from the projection onto the directions inferred by PCA (dark blue) and the proposed method (light blue).}
	\label{fig:sharpnessIndex}
\end{figure}

To evaluate the relevance of the $K$ directions recovered by BNP-PCA, the measured hyperspectral spectra are orthogonally projected on each direction $\Vfac_{1},\dotsc,\Vfac_{K}$. The resulting $K$ images are supposed to explain most of the information contained in the original hyperspectral image with respect to each endmember. They are expected to individually provide relevant interpretation of the scene. The sharpness index introduced by \citet{blanchet2012} as a ground truth-free image quality measure is computed on each image. Figure~\ref{fig:sharpnessIndex} features the corresponding scores for each direction. 
These values are compared with those similarly obtained by a standard PCA. Figure~\ref{fig:sharpnessIndex} shows that our method consistently provides better scores, except for components 3, 5 and 6. This can be empirically explained by the fact that more spatial information (structure and texture) has been recovered by BNP-PCA due to its sparsity promoting property. It ensures a better separation between relevant components and purely random white process than the images projected on the principal components identified by a standard PCA.

\section{Conclusion} \label{sec:conclusion}

This paper indroduces a Bayesian nonparametric principal component analysis (BNP-PCA). This approach permits to infer the orthonormal basis of a latent subspace in which the signal lives as an information distinct from white Gaussian noise. It relies on the use of an Indian buffet process (IBP) prior which permits to deal with a family of models with a potentially infinite number of degrees of freedom. The IBP features two regularizing properties: it promotes sparsity and penalizes the number of degrees of freedom.

Algorithms implementing a Markov chain Monte Carlo (MCMC) sampling are described for all parameters according to their conditional posterior distributions. BNP-PCA appears to be close to completely nonparametric since no parameter tuning or initialization is needed and the most general priors are used. Compared to a parametric approach based on RJ-MCMC, the Markov chain is much easier to implement and mixes much more rapidly. One limitation of the proposed approach is the use of MCMC for inference: faster estimates may be obtained by resorting to variational inference for instance.

Since one may be interested in a BNP approach to estimate the dimension $K$ of the latent subspace (or equivalently the number of degrees of freedom), we have studied the theoretical properties of some estimators based on BNP-PCA in the case where the parameter $\alpha$ of the IBP is fixed. Theorems~\ref{th:inconsistance} \&~\ref{th:inconsistance_severe} show that the marginal MAP (mMAP) estimate of $K$ is not consistent in this case: its posterior does not asymptotically concentrate on any particular value as the number of observations increases.

Numerical experiments show that the proposed BNP-PCA that considers the parameter $\alpha$ of the IBP as an unknown parameter yields very good results. In particular, experimental results indicate that the mMAP estimate of $K$ seems to be consistent (as soon as $\alpha$ is not fixed anymore). To make our approach even more robust, we have elaborated on a Kolmogorov-Smirnov test to propose a method to accurately identify the dimension of the relevant latent subspace. An expected limitation is that a principal component may not be recovered when its energy/eigenvalue is below the noise level.
%
Finally, we have applied BNP-PCA to two classical problems: clustering based on Gaussian models mixture applied to the MNIST dataset and linear unmixing of hyperspectral images (or more generally matrix factorization). The clustering performance of the proposed approach is very good. The inspection of the significance of the elementary images (also called endmembers) estimated from a hyperspectral image is in favour of BNP-PCA compared to standard PCA: each component seems to extract more detailed information as attested by image-guided diagnosis.
 Performed on real datasets, these experiments show that BNP-PCA can be used in a general Bayesian model and yield good performance on real applications. Again we emphasize that the resulting approach will call for very few parameter tuning only.

Based on these encouraging results, future work will aim at studying the consistency of both the new KS-based estimator and the marginal MAP estimator when the IBP parameter has been marginalized. We plan to use BNP-PCA as a subspace identification strategy in a refined linear hyperspectral unmixing method.

\appendix


\section{Marginalized posteriori distribution}
	\label{app:marginalize_coeffs}


\paragraph{}
The marginal posterior distribution is obtained by computing
\begin{equation*}
	f\left( \paramvect,\hypervect  \given \MATobs\right) = \int_{\R^{\dimobs\Nobs} } f\left( \MATobs \given \paramvect,\MATcoef\right) f\left(\paramvect,\MATcoef \given \hypervect \right) f\left(\hypervect\right) \mathrm{d} \MATcoef.
\end{equation*}
The rationale of the proof is to split the exponential in two.
The coefficients $\coefkn$ corresponding to non activated block in $\MATibp$, \textit{i.e.}, for which $\ibpkn=0$, vanish.
The remaining constant is $\prod_{\indcoef=1}^\dimcoef (2\pi \zellnervar_\indcoef)^{-\Vibp_\indcoef\transp\Vibp_\indcoef/2}$ where $\Vibp_\indcoef$ denotes the $\indcoef^{\text{th}}$ row.

\paragraph{}
The remaining exponential term becomes
\begin{equation}
	\label{eq:app_marginal_posteriori:expo_term}
	- \frac{1}{2\noisevar} \sum_{\nobs=1}^{\Nobs} \Bigg( \Vert\Vobs_{\nobs} - \sum_{ \substack{\indcoef \\ \ibpkn=1}} \Vfac_{\indcoef} \Vcoef_{\nobs}\Vert_2^2
	+ \sum_{ \substack{\indcoef \\ \ibpkn=1}} \frac{1}{\zellnervar_\indcoef} \Vcoef_{\nobs}\transp \Vcoef_{\nobs}  \Bigg) .
\end{equation}
The $\ell_2$ norm in Eq.~\eqref{eq:app_marginal_posteriori:expo_term} can be easily simplified since $\Vfac_{l}\transp\Vfac_{m} = \delta_{l, m}$ where $\delta_{l, m}$ is the Kronecker symbol.
In addition, the posterior in Eq.~\eqref{eq:app_marginal_posteriori:expo_term} is conjugated to a Gaussian distribution.
The remaining terms after integration are a constant \\$\big( 2\pi \zellnervar_\indcoef\noisevar /(1 + \zellnervar_\indcoef) \big)^{\Vibp_\indcoef\transp\Vibp_\indcoef/2}$ as well as terms proportional to $\Vobs_{\nobs}\transp \Vfac_{\indcoef} \Vfac_{\indcoef}\transp \Vobs_{\nobs}$ which can be rewritten as $\big(\Vfac_\indcoef\transp \Vobs_\nobs\big)^2$.
The marginal posterior Eq.~\eqref{eq:posterior_marginalized} is obtained by combining all these terms.

\section{Shifted inverse gamma distribution}
	\label{app:sIG}

\begin{figure}
	\includegraphics[width=\columnwidth]{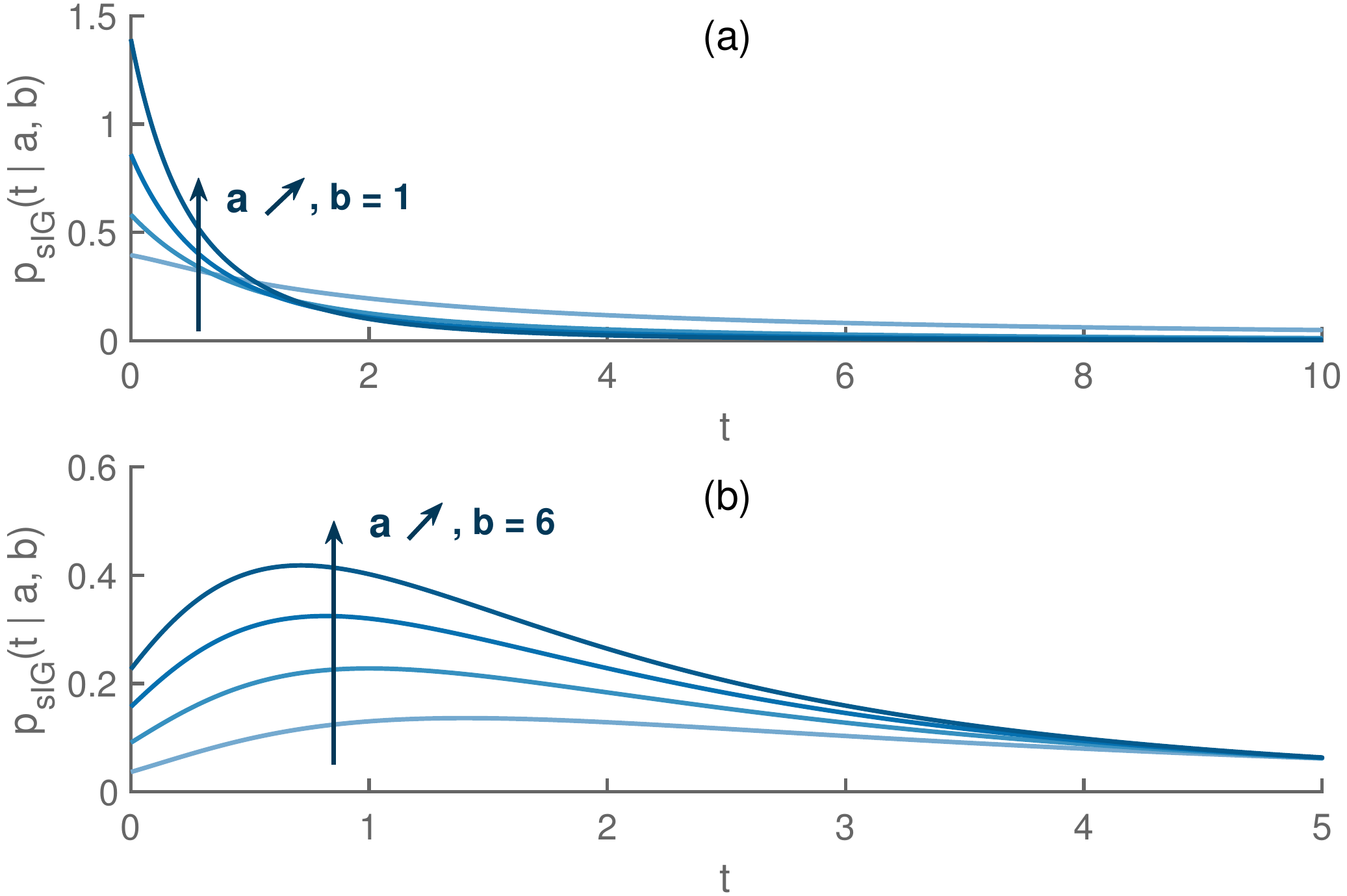}
	\caption{
		\label{fig:app:sIG}
		pdf of the sIG distribution for (a) $a=0.25, 1, 1.5, 2$ and $b=1$,  and (b)  $a=1.5, 2, 2.3, 2.5$ and $b=6$.
	}
\end{figure}

The sIG pdf is defined for all real $x >0$ by
\begin{equation}
	\prob_{\mathrm{sIG}} \big( x  \given a, b \big) = \frac{ b^a }{\gamma \big(a, b \big)}
	\big(1 + x\big)^{-(a+1)} \exp \left( - \frac{b}{1 + x} \right) 
\end{equation}
with shape parameter $a$ and rate parameter $b$, and $\gamma(a,b) = \int_{0}^a t^{b-1}e^{-t}\mathrm{d}\,t$ is the lower incomplete gamma function.
If $b>a+1$, it is easy to see that the pdf has a unique maximum in $\frac{b}{a+1}-1$, but no maximum otherwise.
Fig.~\ref{fig:app:sIG} displays the pdf of the sIG distributions for several values of $a$ and $b$.

if $X \sim \mathrm{sIG}(a, b)$, the two first moments of $X$ are given by
\begin{align}
	\mathbb{E}[X] \,=\,& b \frac{\gamma(a-1, b)}{\gamma(a, b)} - 1 \\
	\mathrm{var}(X) \,=\,& b^2 \Bigg( \frac{\gamma(a-2, b)}{\gamma(a, b)} - \left(\frac{\gamma(a-1, b)}{\gamma(a, b)}\right)^2 \Bigg)  .
\end{align}
Note finally that the sIG distribution can be easily sampled by resorting to the change of variable $u=1+\zellnervar_\indcoef$ where $u^{-1}$ follows a Gamma distribution of parameters $a_\zellner$ and $b_\zellner$  truncated on the segment $(0, 1)$.

\section{Jeffreys' prior for the IBP hyperparameter}
	\label{app:ibp_jeffrey}

By definition, the Jeffreys' prior is given by \citep[Ch. 2]{Marin2007}
\begin{equation}
	f(\ibpparam) \propto \sqrt{ \mathbb{E} \bigg[ \Big(\frac{\mathrm{d}}{\mathrm{d}\,\ibpparam} \log \Prob \big[ \MATibp | \ibpparam \big] \Big)^2 \bigg] }.
\end{equation}
Since $\frac{\mathrm{d}}{\mathrm{d}\,\ibpparam} \log \Prob \big[ \MATibp | \ibpparam \big] = \frac{\dimcoef}{\ibpparam} - \sum_{\nobs=1}^\Nobs \frac{1}{\nobs}$, and does not depend on $\MATibp$,
\begin{equation}
	\mathbb{E} \Bigg[ \Big(\frac{\mathrm{d}}{\mathrm{d}\,\ibpparam} \log \Prob \big[ \MATibp | \ibpparam \big] \Big)^2 \Bigg] = \Bigg( \frac{\dimcoef}{\ibpparam} - \sum_{\nobs=1}^\Nobs \frac{1}{\nobs} \Bigg)^2 .
\end{equation}
Thus $f(\ibpparam) \propto \ibpparam^{-1}$.



\section{Marginalized posterior distribution}
	\label{app:marginalizedPosteriorForMapm}

The marginal posterior distribution is obtained by integrating the marginal posterior given by Eq.~\eqref{eq:posterior_marginalized} with respect to the parameters $\zellnervar$ and $\ibpparam$. By mean of conjugacy, straightforward computations lead to
\begin{equation}
	\begin{split}
		f&\left( \MATfac, \MATibp, \noisevar \given \MATobs \right) =
		\left( \frac{1}{2\pi \noisevar} \right)^{\Nobs\dimobs/2} \exp \left[ \mathrm{trace}\left[ - \frac{1}{2\noisevar}\MATobs\MATobs\transp \right] \right] \\
		&\times \left( \frac{b_\delta^{a_\delta}}{\gamma(a_\delta, b_\delta)} \right)^\dimcoef \prod_{\indcoef=1}^\dimcoef \frac{ \gamma(a_\indcoef, b_\indcoef) }{b_\indcoef^{a_\indcoef}} \exp \left( \frac{1}{2\noisevar} \sum_\nobs \scalarProdSquare{\Vfac_\indcoef}{\Vobs_\nobs} \right) \\
		&\times \left( \sum_\nobs \frac{1}{\nobs} \right)^{-\dimcoef} \frac{  \Gamma(\dimcoef) }{ \prod_\indcoef \dimcoef_{\nobs}! } \prod_\indcoef \frac{(\Nobs - m_{\indcoef})! \; (m_\indcoef-1)!}{\Nobs!} \indfunc{ \mathbb{U}_{\dimobs} }{ \MATfac },
	\end{split}
\end{equation}
where for all $\indcoef$
\begin{align*}
	a_\indcoef {}={}& a_\delta + \Vibp_\indcoef\transp\Vibp_\indcoef \\
	b_\indcoef {}={}& b_\delta + \frac{1}{2\noisevar} \sum_\nobs \scalarProdSquare{\Vfac_\indcoef}{\Vobs_\nobs} .
\end{align*}

\section{Law and expectation of scalar product}
	\label{app:expectation_scalar_product}

{
\newcommand{\vinit}{\bsu}
\newcommand{\palea}{\bsv}
\newcommand{\qalea}{\bsq}

This section derives the marginal distribution of the projections evoked in Theorem~\ref{th:distribution_scalar_product} under the uniform distribution over $\stiefel{\dimobs-\dimcoef}{\dimobs}$.

\paragraph{}
\noindent\textbf{Area element of the sphere.}
The rationale of the proof is to adapt the vector to the area element in the $\dimobs$-dimensional Euclidean space expressed in spherical coordinate. The $\dimobs$-dimensional element parametrized by $\dimobs-1$ angles is given by
\begin{equation*}
	\mathrm{d}{}S^{\dimobs} = \sin^{\dimobs-2}(\phi_1) \sin^{\dimobs-3}(\phi_2) \dots \sin(\phi_{\dimobs-2}) \mathrm{d}{}\phi_1 \dots \mathrm{d}{}\phi_{\dimobs-1},
\end{equation*}
and the Cartesian coordinates $\palea_1\dots \palea_\dimobs$ of a vector $\palea$ are given by
\begin{align*}
	\palea_1\ =& \cos (\phi_1) \\
	\palea_2\ =& \sin(\phi_1) \cos(\phi_2) \\
	\vdots& \\
	\palea_{\dimobs-1}\ =& \sin(\phi_1) \dots \sin(\phi_{\dimobs-2})\cos (\phi_{\dimobs-1}) \\
	\palea_{\dimobs}  \ =& \sin(\phi_1) \dots \sin(\phi_{\dimobs-1}).
\end{align*}
The proof considers a non explicit rotation applied to $\vinit$ such that only the last component $\palea_{\dimobs}$ is involved in the scalar product.



\paragraph{}
\noindent\textbf{Proof.}
Let $\bfu$ be a unit vector of $\mathbb{R}^L$.
See $L$ here as the size of the orthogonal of the relevant component, $L=\dimobs-\dimcoef$.
Let $\boldsymbol{\nu}$ be a random variable uniformly distributed on the $L$-dimensional unit sphere.
Let also $w$ be the random variable associated to the scalar product $w =  |\langle \bfu, \boldsymbol{\nu} \rangle | =|\boldsymbol{\nu}\transp \bfu|$.
The density of $w$ will be obtain from the cdf
\begin{align}
	\mathrm{p}_w \left( w \leq \lambda \right) {}={}& \mathrm{p}_{\boldsymbol{\nu}} \left( \vert \boldsymbol{\nu}\transp\bfu  \vert \leq \lambda \right)
	={} \int \bm{1}_{\vert \boldsymbol{\nu}\transp \bfu \vert}(\boldsymbol{\nu}) \mathrm{d} \boldsymbol{\nu} , \label{eq:app:expectation:cdf_beginning}
\end{align}
where the sum appearing in the last equation is expressed w.r.t. the Haar measure on the sphere.
\par

Let $\bfR$ the rotation matrix such that $\bfe = \bfR \bfu$ where ${\bse} = [1, 0, 0, \dots]$.
Since the Haar measure is invariant under rotation, Eq.~\eqref{eq:app:expectation:cdf_beginning} becomes, once rewritten w.r.t. the area element $\mathrm{d}S^{L-1}$
\begin{equation*}
	\mathrm{p} \left( w \leq \lambda \right) =
	\frac{1}{\mathcal{S}_{L-1}} \int \bm{1}_{\vert \cos(\phi_1) \vert \leq \lambda}(\bsv) \mathrm{d} S^{L-1}.
\end{equation*}

Since $\vert \cos(\phi_1) \vert \leq \lambda$ if $\phi_1$ belongs to the set $[\arccos(\lambda), \pi - \arccos(\lambda)]$, one have, by means of symmetry around $\pi/2$
\begin{align*}
	\mathrm{p} \left( w \leq \lambda \right) &{}={}
	\frac{2}{\mathcal{S}_{L-1}} \int_{\phi_1 = \arccos(\lambda)}^{\pi/2} \int_{\phi_2\dots \phi_{L-2}=0}^\pi \int_{\phi_{L-1}=0}^{2\pi} \\
	& \sin^{L-2}(\phi_1) \dots \sin(\phi_{L-2}) \mathrm{d}\phi_1 \dots \mathrm{d}\phi_{L-1} \\
	&{}={} 2\frac{\mathcal{S}_{L-2}}{\mathcal{S}_{L-1}} \int_{\phi_1 = \arccos(\lambda)}^{\pi/2} \sin^{L-2}(\phi_1)  \mathrm{d}\phi_1,
\end{align*}
which is only composed of independent sum.
By recognizing the area of the $L-2$-sphere and by defining the change of variable $y = \cos(\phi_1)$, one have
\begin{equation*}
	\mathrm{p} \left( w \leq \lambda \right) {}={} \frac{\mathcal{S}_{L-2}}{\mathcal{S}_{L-1}} \; 2\int_0^\lambda \sin^{L-3}(\arccos(y)) \mathrm{d}y .
\end{equation*}

Knowing that $\sin(\arccos(y))$ can be rewritten as $\sqrt{1 - y^2}$, one obtains, after two changes of variable
\begin{align*}
	\int_0^\lambda \sin^{L-3}(\arccos(y)) \mathrm{d}y
	&{}={} \int_0^\lambda \left(1 - y^2\right)^{L-3} \mathrm{d}y \\
	&{}={} \lambda \int_0^1 \left(1 - \lambda^2y^2\right)^{L-3} \mathrm{d}y \\
	&{}={} \frac{\lambda}{2} \int_0^1 \left(1 - \lambda^2z\right)^{L-3} z^{-1/2} \mathrm{d}z .
\end{align*}
The sum can be resolved using Corollary~1.6.3.2 page~36 in \cite{gupta1999} with parameters $\alpha=\frac{1}{2}$, $\beta=-\frac{L-3}{2}$, $\gamma=\frac{3}{2}$ and $R=\lambda^2$, leading to
\begin{equation*}
	\int_0^\lambda \sin^{L-3}(\arccos(y)) \mathrm{d}y = 2 \lambda \; _2F_1\left( \frac{1}{2}, -\frac{L-3}{2} ; \frac{3}{2} ; \lambda^2 \right),
\end{equation*}
which is the expected result.

}

\section{Inconsistency of the marginal MAP estimator of the latent dimension}
\label{app:inconsistency}

{
\newcommand{\setA}{\calA}
\newcommand{\setB}{\calB}
We emphasize that the proof is conducted with arguments similar to the one in \cite{miller2013}.
\par

Let first introduce a few notations.
We call $\setA(\dimcoef, \Nobs)$ the set all binary matrices $\MATibp$ with $\dimcoef$ rows and $\Nobs$ columns.
For every binary matrix $\MATibp$, we call $\setB(\MATibp)$ the set of matrices $\MATibp'$ which are identical to $\MATibp$ except that a new line have been added with only one active element.
The notation $\MATibp'(j)$ will seldom be employed, where $j$ indicates the index of the new active element.
Finally, let $c_\Nobs(\dimcoef, \alpha)$ be the quantity
\begin{equation}
	c_\Nobs(\dimcoef, \alpha) \overset{\triangle}{=} \underset{\MATibp \in \setA(\dimcoef, \Nobs)}{\max} \quad \underset{\MATibp' \in \setB(\MATibp)}{\max} \quad \frac{\mathrm{P}[\MATibp | \alpha ]}{\mathrm{P}[\MATibp' | \alpha ]} .
\end{equation}
\par

\subsection{Two lemmas}

Let first consider the two following lemmas
\begin{lemma} \label{lem:app:consistance:prior}
	For all $\alpha, \dimcoef$
	\begin{equation}
		\underset{\Nobs \rightarrow + \infty}{\lim\;\sup} \quad \frac{1}{\Nobs} c_\Nobs(\dimcoef, \alpha) \leq + \infty.
	\end{equation}	
\end{lemma}

\begin{proof}
	Let $\Nobs, \dimcoef$ be two positive integers, $\MATibp, \MATibp'$ two  binary matrices belonging respectively to $\setA(\dimcoef, \Nobs)$ and $\setB(\MATibp)$.
	\par
	According to Eq.~\eqref{eq:intro:pdf_ibp}, one have, by noting $K_{\mathrm{new}}^h$ the number of column in $\MATibp'$ identical to the added one,
	\begin{equation*}
		\frac{\mathrm{P}[\MATibp | \alpha ]}{\mathrm{P}[\MATibp' | \alpha ]} \leq \frac{\Nobs}{\alpha} K_{\mathrm{new}}^h
		{}\leq{} \frac{\dimcoef}{\alpha} \Nobs ,
	\end{equation*}
	which lead to the expected result.
	\qed
\end{proof}

\begin{lemma} \label{lem:app:consistance:majorationLLK}
	Let $\MATibp, \MATibp'$ be respectively two elements of $\setA(\dimcoef, \Nobs)$ and $\setB(\MATibp)$. 
	Thus,
	\begin{equation}
		\mathrm{p}\left( Y_{1:N} \mid \MATibp \right) \leq \kappa \; \mathrm{p}\left( Y_{1:N} \mid \MATibp'  \right) ,
	\end{equation}
	where
	\begin{equation}
		\kappa = b_{\delta} \frac{\gamma(a_\delta, b_\delta)}{\gamma(a_\delta+1, b_\delta)} .
	\end{equation}
\end{lemma}

\begin{proof}
	Let $\Theta$ be the set of all parameters and hyperperameters, such that
	\begin{equation*}
		\mathrm{p} (\MATobs | \MATibp) = \int_{\Theta} \mathrm{p} (\MATobs | \theta, \MATibp) \mathrm{p}(\theta | \MATibp) \mathrm{d}\theta .
	\end{equation*}
	Let $\MATibp'$ be an element of $\setB(\MATibp)$, and $j$ be the index of the active element in the new line.
	Note the activation of the $j^{th}$ element adds a term of the form
	\begin{equation} \label{eq:app:consistance:proof_lemma2}
		\frac{1}{1+\delta_{K+1}^2} \exp \left( \frac{\delta_{K+1}^2}{1+\delta_{K+1}^2} \frac{ \scalarProdSquare{\Vobs_j}{\Vfac_{K+1}}}{\noisevar} \right) .
	\end{equation}
	The term in the exponential is always positive, so the exponential can be minored by 1. By integrating w.r.t. $\delta_{K+1}^2$, one has
	\begin{align*}
		\mathrm{p}\left( Y_{1:N} \mid \MATibp'  \right) \geq \frac{b_\delta^{a_\delta}}{\gamma(a_\delta, b_\delta)} \frac{\gamma(a_\delta+1, b_\delta)}{b_\delta^{a_\delta+1}} \mathrm{p}\left( Y_{1:N} \mid \MATibp \right),
	\end{align*}
	which completes the proof.
	\qed
\end{proof}

\subsection{proof}

For all integer $j$ in $\llbracket 1, \Nobs\rrbracket$
\begin{align}
	\mathrm{p}\big( \MATobs,& \dimcoef_\Nobs = \dimcoef | \alpha \big)  \\
	{}={}& \sum_{\MATibp_\dimcoef \in \setA(\dimcoef, \Nobs)}  \mathrm{P}\left[\MATibp_\dimcoef\right] \mathrm{p}(\MATobs | \MATibp_\dimcoef, \alpha) \nonumber \\
	{}\leq{}& \sum_{\MATibp_\dimcoef \in \setA(\dimcoef, \Nobs)}  {\Nobs c_\Nobs(\dimcoef, \alpha)\; \mathrm{P}\left[\MATibp'(j)\mid \alpha\right]} \; {\kappa \;\mathrm{p}\left(\MATobs \mid {\MATibp'}(j), \alpha \right)} . \nonumber
\end{align}
where the last inequality has been obtained using both Lemmas~\ref{lem:app:consistance:prior} and \ref{lem:app:consistance:majorationLLK}.
Since this inequality is true for all $j$, one can average over all values of $j$, leading to
\begin{align*}
	\mathrm{p}\big( &\MATobs, \dimcoef_\Nobs = \dimcoef | \alpha \big)  \\
	{}\leq{}& \sum_{\MATibp_\dimcoef \in \setA(\dimcoef, \Nobs)}
		\sum_{j=1}^\Nobs
	\kappa c_\Nobs(\dimcoef, \alpha)\; \mathrm{P}\left[\MATibp'(j)\mid \alpha\right] \; \;\mathrm{p}\left(\MATobs \mid {\MATibp'}(j), \alpha \right) \\
	{}\leq{}& \kappa c_\Nobs(\dimcoef, \alpha)
	\sum_{\MATibp_\dimcoef \in \setA(\dimcoef, \Nobs)}
	\sum_{\MATibp'\in \setA(\dimcoef+1, \alpha)} \mathrm{p}\left(\MATobs | \MATibp'(j) | \alpha \right) \bm{1}_{\MATibp' \in \setB(\MATibp)} \\
	{}\leq{}& \kappa c_\Nobs(\dimcoef, \alpha)
	\sum_{\mathclap{\MATibp'\in \setA(\dimcoef+1, \alpha)}} \mathrm{card}\Big\{ \MATibp, \MATibp' \in \setB(\MATibp) \Big\}
	\mathrm{p}\left(\MATobs | \MATibp'(j) | \alpha \right) \bm{1}_{\MATibp' \in \setB(\MATibp)} .
\end{align*}

However, for each matrix $\MATibp'$ in $\setA(\dimcoef+1, \alpha)$, there are at most one matrix $\MATibp$ verifying the condition, leading to
\begin{equation}
	\begin{split}
		\mathrm{p}\big( &\MATobs, \dimcoef_\Nobs = \dimcoef | \alpha \big)  \\
		{}\leq{}& \kappa c_\Nobs(\dimcoef, \alpha)
		\sum_{\MATibp'\in \setA(\dimcoef+1, \alpha)}
		\mathrm{p}\left(\MATobs | \MATibp'(j) | \alpha \right) \bm{1}_{\MATibp' \in \setB(\MATibp)} . \label{eq:app:consistance:inequality_almost_finish}
	\end{split}
\end{equation}

From now, the proof is almost finished.
By the Bayes rule, one has for $\dimcoef < \dimobs$
\begin{align*}
	\mathrm{p} \Big( \dimcoef_\Nobs&=\dimcoef | \MATobs,  \alpha \Big) \\
	{}={}& \frac{ \mathrm{p} \left( \dimcoef_\Nobs=\dimcoef, \MATobs \mid \alpha \right) }{ \sum_{k=0}^\infty \mathrm{p} \left( \dimcoef_\Nobs=\dimcoef, \MATobs, \alpha \right) } \\
	{}<{}& \frac{ \mathrm{p} \left( \dimcoef_\Nobs=\dimcoef, \MATobs | \alpha \right) }{ \mathrm{p} \left( \dimcoef_\Nobs=\dimcoef, \MATobs, \alpha \right) + \mathrm{p} \left( \dimcoef_\Nobs=\dimcoef + 1 | \MATobs, \alpha \right) } \\
	{}<{}& \frac{c_\Nobs(\dimcoef, \alpha) \kappa}{c_\Nobs(\dimcoef, \alpha) \kappa + 1} \\
	{}<{}& 1 .
\end{align*}

finally, for $\dimcoef = \dimobs$
\begin{align*}
	\mathrm{p} \Big( \dimcoef_\Nobs=&\dimobs | \MATobs, \alpha \Big) \\
	{}={}& \frac{ \mathrm{p} \left( \dimcoef_\Nobs=\dimobs, \MATobs | \alpha \right) }{ \sum_{k=0}^\infty \mathrm{p} \left( \dimcoef_\Nobs=k \mid \MATobs, \alpha \right) } \\
	{}\geq{}& \frac{ \mathrm{p} \left( \dimcoef_\Nobs=\dimobs, \MATobs | \alpha \right) }{ \sum_{k=0}^\dimobs \left(c_\Nobs(k, \alpha)\kappa\right)^{\dimcoef-k} \mathrm{p} \left( \dimcoef_\Nobs=k | \MATobs, \alpha \right) } \\
	{}\geq{}& \frac{ 1 }{ \sum_{k=0}^\dimobs \left({c_\Nobs(k, \alpha)\kappa}\right)^{\dimobs-k} } \\
	\geq{}& \frac{ 1 }{ 1 + \sum_{k=1}^\dimobs \left({c_\Nobs(\dimcoef, \alpha)\kappa}\right)^{\dimcoef} } \\
	{}>{}& 0.
\end{align*}

One can see from the last couple of equations that the result stated in Eq.~\eqref{eq:th:consistance:estimator:K_k} can be generalized to all models based on an IBP and verifying Lemma~\ref{lem:app:consistance:majorationLLK}.
However, the result in Eq.\eqref{eq:th:consistance:estimator:K_D} results from the orthogonality constraints.

}

\section{Severe inconsistency in case of a simple generative model}
\label{app:severe_inconsistency}

Assumes that for all $\nobs$, $\Vobs_\bsn \sim \calN(0, \noisevar \indmat_\dimobs)$ and $g$ be the quantity
\begin{equation*}
	\begin{split}
	g(\MATobs, \MATibp, \MATfac, \delta^2) = \frac{\calK(a_{\delta^2}, b_{\delta^2})^\dimcoef}{\mathrm{vol}(\calS_D)} \prod_{\indcoef=1}^{\dimcoef} \left(\frac{1}{1 + \delta_k^2} \right)^{a_{\delta^2} + \bsz_\indcoef\transp\bsz_\indcoef} \\
	\exp \left[ - \frac{1}{1 + \delta_\indcoef^2} \left( b_{\delta^2} + \frac{1}{2\noisevar} \sum_{\nobs=1}^\Nobs \ibpkn \scalarProdSquare{\Vfac_\indcoef}{\Vobs_\nobs} \right) \right],
	\end{split}
\end{equation*}
\textit{i.e.}, $ g\propto \mathrm{p}\left(\MATibp, \MATfac, \delta^2 | \MATobs, \noisevar, \alpha \right)$.
Let emphasize that $g$ is intimately linked to a probability distribution.
\par

Let $K_\Nobs$ be again the random variable associated to the latent subspace dimension.
One has, by definition
\begin{align}
	\mathrm{P}\left[ K_\Nobs = 0 | \MATobs, \noisevar, \alpha \right] {}={}& \frac{ \mathrm{p}\left( K_\Nobs = 0, \MATobs | \noisevar, \alpha \right) }{ \sum_{K=1}^{+\infty} \mathrm{p}\left( K_\Nobs = K, \MATobs | \noisevar, \alpha \right) } \nonumber \\
	{}\leq{}& \frac{1}{1 + \frac{\mathrm{p}\left( K_\Nobs = 1, \MATobs | \noisevar, \alpha \right)}{\mathrm{p}\left( K_\Nobs = 0, \MATobs | \noisevar, \alpha \right)}} \label{eq:app:inconsistency_with_gen_model:frac} .
\end{align}

The quantity appearing in the denominator of Eq.~\eqref{eq:app:inconsistency_with_gen_model:frac} can be rewritten
\begin{align*}
	\;&\frac{\mathrm{p}\left( K_\Nobs = 1, \MATobs | \noisevar, \alpha \right)}{\mathrm{p}\left( K_\Nobs = 0, \MATobs | \noisevar, \alpha \right)} \\
	{}={}& \sum_{\MATibp, K_\Nobs=1} \int_{\calS_\dimobs} \int_{\mathbb{R}_+} g(\MATobs, \MATibp, \delta^2, \MATfac) \mathrm{d}\noisevar \mathrm{d}\MATfac \mathrm{d}\delta^2 \; \frac{\mathrm{P}[\MATibp | \alpha]}{\mathrm{P}[\bm{0} | \alpha]} \\
	{}={}& \sum_{\MATibp, K_\Nobs=1} \int_{\calS_\dimobs} \int_{\mathbb{R}_+} g(\MATobs, \MATibp, \delta^2, \MATfac) \mathrm{d}\noisevar \mathrm{d}\MATfac \mathrm{d}\delta^2 \\
	&\times  \; \alpha \frac{(\Nobs - \Vibp_1^t\Vibp_1)!(\Vibp_1^t\Vibp_1 - 1)!}{N!} .
\end{align*}

Since the matrix $\MATibp$ appearing in the former equation has only one row, one can decompose the sum over the number of active component and the number of instance,
\begin{align*}
	&\frac{\mathrm{p}\left( K_\Nobs = 1, \MATobs | \noisevar, \alpha \right)}{\mathrm{p}\left( K_\Nobs = 0, \MATobs | \noisevar, \alpha \right)} {}={}\\
	& \sum_{l=1}^\Nobs \sum_{\MATibp, K_\Nobs=1, \Vibp_1\Vibp_1\transp = l}
	\frac{\alpha}{l} \frac{1}{\binom{\Nobs}{l}}
	\int_{\calS_\dimobs} \int_{\mathbb{R}_+} g(\MATobs, \MATibp, \delta^2, \MATfac) \mathrm{d}\noisevar \mathrm{d}\MATfac \mathrm{d}\delta^2 .
\end{align*}

Let define for each $l$ the U-statistic
\begin{equation}
	U_l(\bsY) \overset{\triangle}{=} \frac{1}{\binom{\Nobs}{l}} \quad
	\sum_{\mathclap{\substack{\MATibp, K_\Nobs=1, \\ \Vibp_1\Vibp_1\transp = l}}} \quad
	\int_{\calS_\dimobs\cup\mathbb{R}_+} g(\MATobs, \MATibp, \delta^2, \MATfac) \mathrm{d}\noisevar \mathrm{d}\MATfac \mathrm{d}\delta^2,
\end{equation}
where the support of each permutation is given by the $l$ active components of $\bsZ$. By the strong law of large number \citep{hoeffding1961}, for all $l$,
\begin{equation}
	U_l(\bsY) \overset{a.s.}{\underset{\Nobs\rightarrow+\infty}{\longrightarrow}} \mathbb{E}_{\bsY}  \left[ \int_{\calS_\dimobs \cup \mathbb{R}_+} g(\MATobs, \MATibp, \delta^2, \MATfac) \mathrm{d}\noisevar \mathrm{d}\MATfac \mathrm{d}\delta^2 \right] = 1. \label{eq:app:inconsistency_with_gen_model:hoffding}
\end{equation}
The former equality holds since the quantity under the expectation is a density.
Consequently, for all $L\leq \Nobs$
\begin{align*}
	\frac{\mathrm{p}\left( K_\Nobs = 1, \MATobs | \noisevar, \alpha \right)}{\mathrm{p}\left( K_\Nobs = 0, \MATobs | \noisevar, \alpha \right)}
	{}\geq{}& \sum_{l=1}^L \frac{\alpha}{l} U_l(\bsY) \overset{a.s.}{\underset{\Nobs\rightarrow+\infty}{\longrightarrow}} \sum_{l=1}^L \frac{\alpha}{l} .
\end{align*}
Since the former equality is true for all $L$, and that the harmonic series $\sum_l \frac{1}{l}$ diverges, the quantity $\frac{\mathrm{p}\left( K_\Nobs = 1, \MATobs | \noisevar, \alpha \right)}{\mathrm{p}\left( K_\Nobs = 0, \MATobs | \noisevar, \alpha \right)}$ goes to infinity almost surely as $\Nobs$ increases. This complete the proof.

\section{Marginal posterior distribution of the scale parameters}\label{app:deltapost}
\newcommand{\scalef}{\lambda}
In the general case, the posterior distribution of the scale parameters $\boldsymbol{\zellner} = \left\{\zellnervar_1,\ldots,\zellnervar_K\right\}$, where the orthogonal matrix $\MATfac$ has been marginalized, cannot be derived analytically. However, assuming that the binary matrix $\MATibp$ is the $\dimcoef\times\Nobs$ matrix $\boldsymbol{1}_{\dimcoef,\Nobs}$ with only $1$'s everywhere, this posterior distribution can be derived explicitly. In particular, when $K=D$
\begin{equation} \label{eq:algo:theoretical_behaviour}
	\begin{split}
	&f \left(\boldsymbol{\zellner} \given \MATobs, \noisevar, \ibpparam, \MATibp = \boldsymbol{1}_{\dimobs,\dimobs} \right) \propto\\
        &\prod_{\indcoef = 1}^{\dimobs} \left(\frac{1}{1 + \zellnervar_\indcoef}\right)^{a_{\delta}+1} \exp \left( - \frac{b_{\delta}}{1+ \zellnervar_\indcoef} \right)
         \\
		& \times {} _0\mathrm{F}_0\left( \emptyset, \emptyset, \frac{1}{\noisevar} \MATobs \MATobs\transp - \scalef \indmat_\dimobs, \boldsymbol{\Delta}_{\boldsymbol{\zellner}} \right) \etr\left( \scalef \boldsymbol{\Delta}_{\boldsymbol{\zellner}} \right)
	\end{split}
\end{equation}
with $\scalef \in (0,\frac{1}{\sigma^2} \rho_{\min})$ where\footnote{Note that the positive real number $\scalef$ has no particular interpretation and is only introduced here for convenience.} $\rho_{\min}$ is the minimum eigenvalue of $\MATobs \MATobs\transp$, $\boldsymbol{\Delta}_{\boldsymbol{\zellner}}$ is a $\dimobs\times\dimobs$ diagonal matrix formed by the ratios $\zellnervar_\indcoef / (1 + \zellnervar_\indcoef)$ and $_0\mathrm{F}_0$ is a generalized hypergeometric function of two matrices. In particular, this function is defined by
\begin{equation}
	{}_0\mathrm{F}_0 (\emptyset, \emptyset, \bfA, \bfB) = \sum_{\indcoef=1}^\infty \sum_{\kappa \vdash \indcoef} \frac{ C_\kappa(\bfA) C_\kappa(\bfB) }{ C_\kappa(\indmat_D) k!}
\end{equation}
where $\kappa \vdash  \indcoef$ denotes the integer partitions of $\indcoef$, $C_\kappa(\bfA)$ is a zonal polynomial defined by the eigenvalues of $\bfA$ \cite[Ch.~7]{muirhead1982}. Despite recent advances in numerical evaluation of zonal polynomials due to, e.g., \cite{Koev2006}, this quantity remains difficult to be computed. However, it can be interpreted as a measure of mismatch between the magnitudes of the principal components recovered by PCA (through the eigenvalues of $\frac{1}{\sigma^2}\MATobs \MATobs\transp - \scalef \indmat_\dimobs$) and the magnitudes of the relevant components identified by the proposed procedure (in $\boldsymbol{\Delta}_{\boldsymbol{\zellner}}$).

More generally, this hypergeometric function can be advocated for as an elegant way to compare two positive definite matrices using their respective eigenvalues. This finding would suggest the design of an appropriate metric which allows two covariance matrices  to be compared regardless of their respective induced orientations.

\bibliographystyle{spbasic} 

\begin{acknowledgement}
	Part of this work has been funded thanks to the BNPSI ANR project no. ANR-13-BS-03-0006-01.
\end{acknowledgement}

\end{document}